\PassOptionsToPackage{round}{natbib}
\documentclass{article}

\usepackage[preprint]{neurips_2026}

\usepackage[utf8]{inputenc}
\usepackage[T1]{fontenc}
\usepackage{url}
\usepackage{booktabs}
\usepackage{amsfonts}
\usepackage{amsmath}
\usepackage{amssymb}
\usepackage{amsthm}
\usepackage{nicefrac}
\usepackage{microtype}
\usepackage{xcolor}
\usepackage{algorithm}
\usepackage{algorithmic}
\usepackage{graphicx}
\usepackage{float}
\usepackage{multirow}
\usepackage{pifont}
\usepackage{hyperref}
\usepackage{listings}
\definecolor{codebg}{gray}{0.96}
\definecolor{codegreen}{rgb}{0.0, 0.5, 0.0}
\definecolor{codepurple}{rgb}{0.58, 0.0, 0.82}
\definecolor{codeorange}{rgb}{0.8, 0.36, 0.0}
\definecolor{codeblue}{rgb}{0.0, 0.33, 0.71}
\definecolor{codeteal}{rgb}{0.0, 0.45, 0.45}
\definecolor{codered}{rgb}{0.75, 0.0, 0.0}
\lstdefinestyle{code}{
  language=Python,
  basicstyle=\ttfamily\footnotesize,
  backgroundcolor=\color{codebg},
  keywordstyle=\color{codepurple}\bfseries,
  keywordstyle={[2]\color{codeblue}\bfseries},
  commentstyle=\color{codegreen}\itshape,
  stringstyle=\color{codeorange},
  numberstyle=\tiny\color{gray},
  morekeywords={[2]True,False,None},
  emph={[1]jnp,jax,torch,F,nn,lax},
  emphstyle={[1]\color{codeblue}},
  emph={[2]softmax,log_softmax,stop_gradient,detach,mean,std,where,log,sum},
  emphstyle={[2]\color{codeteal}},
  emph={[3]tpo_skill,tpo_target},
  emphstyle={[3]\color{codered}},
  literate=
    {1e-6}{{{\color{codeteal}1e-6}}}4
    {1e-8}{{{\color{codeteal}1e-8}}}4
    {(-1)}{{({\color{codeteal}-1})}}4
    {, -1)}{{, {\color{codeteal}-1})}}5
    {, -1,}{{, {\color{codeteal}-1},}}5,
  numbers=left,
  numbersep=8pt,
  frame=lines,
  framesep=2mm,
  breaklines=true,
  showstringspaces=false,
  keepspaces=true,
  columns=fullflexible,
  upquote=true,
}
\usepackage{tikz}
\usetikzlibrary{positioning,calc,arrows.meta}

\definecolor{bestcol}{RGB}{46,160,67}
\definecolor{medcol}{RGB}{245,140,32}
\definecolor{badcol}{RGB}{221,60,60}

\tikzset{
  >={Latex[length=2mm]},
  heading/.style={font=\bfseries},
  subheading/.style={font=\scriptsize, text=black!65},
  annot/.style={font=\scriptsize, text=black!70, align=center},
  fitbox/.style={
    draw, rounded corners=3pt, thick, fill=gray!6,
    minimum height=0.7cm, align=center, inner sep=4pt,
    font=\scriptsize\bfseries
  },
  flow/.style={->, thick},
}
\newtheorem{proposition}{Proposition}

\title{Target Policy Optimization}

\author{%
  Jean Kaddour
}

\begin{document}

\maketitle

\begin{figure}[H]
\centering
\includegraphics[width=\textwidth]{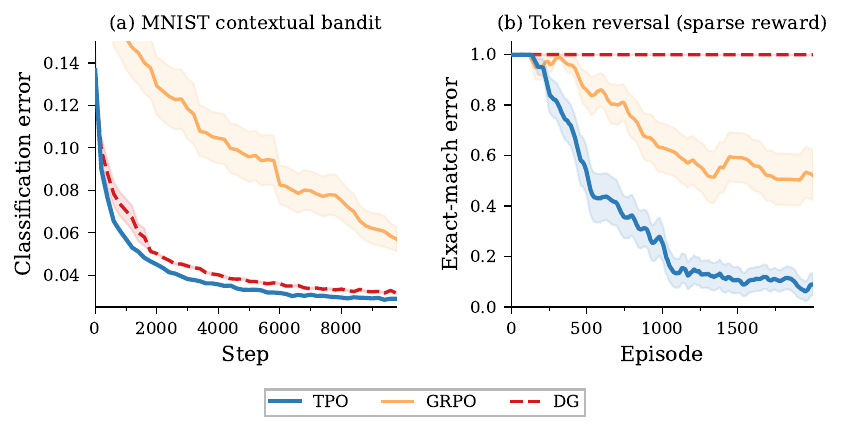}
\caption{\textbf{TPO matches baselines on easy tasks and outperforms them under sparse reward.} (a)~On an MNIST contextual bandit with dense reward, TPO converges slightly faster than GRPO and DG. (b)~On a sparse-reward token-reversal task (reward only at end of sequence), GRPO and DG stall near random while TPO solves the task. Both panels show mean $\pm$ s.e.\ over 20 seeds.}
\label{fig:hero}
\end{figure}

    \begin{abstract}
        In RL, given a prompt, we sample a group of completions from a model and score them. Two questions follow: which completions should gain probability mass, and how should the parameters move to realize that change? Standard policy-gradient methods answer both at once, so the update can overshoot or undershoot depending on the learning rate, clipping, and other optimizer choices. We introduce \emph{Target Policy Optimization} (TPO), which separates the two questions. Given scored completions, TPO constructs a target distribution $q_i \propto p_i^{\,\mathrm{old}} \exp(u_i)$ and fits the policy to it by cross-entropy. The loss gradient on sampled-completion logits is $p^\theta - q$, which vanishes once the policy matches the target. On tabular bandits, transformer sequence tasks, and billion-parameter LLM RLVR, TPO matches PG, PPO, GRPO, and DG on easy tasks and substantially outperforms them under sparse reward. Code is available at \url{https://github.com/JeanKaddour/tpo}.         \end{abstract}

\section{Introduction}
\label{sec:intro}

Consider a prompt for which we sample a small group of candidate completions from a model and score them. We want to shift probability mass toward the better completions. Standard policy-gradient methods entangle the desired redistribution with the optimizer mechanics that realize it. This coupling can make learning fragile, especially when reward is sparse (Figure~\ref{fig:hero}).

A natural fix is to decouple the two questions: first construct a target distribution that encodes the desired redistribution, then fit the policy to it. This reweight-then-fit idea dates to \citet{dayan1997using} and has been instantiated by REPS \citep{peters2010reps} and MPO \citep{abdolmaleki2018mpo}, but those methods require learned Q-functions and constrained optimization over action spaces.

We propose \emph{Target Policy Optimization} (TPO), which applies the same principle to the finite candidate sets used in group-based RL. In this setting, the target distribution is available in closed form, without a critic or dual optimization. Given the probabilities $p_i^{\mathrm{old}}$ assigned by the behavior policy and standardized scores $u_i$, TPO constructs $q_i \;\propto\; p_i^{\mathrm{old}} \exp(u_i),$ then fits the policy to $q$ by cross-entropy. The gradient vanishes exactly when the policy matches the target.

We evaluate TPO on exact tabular bandits, MNIST contextual bandits, sparse-reward transformer tasks, and LLM RLVR. It matches policy-gradient baselines on easier tasks and outperforms them where reward is sparse.

\section{Target Policy Optimization}
\label{sec:tpo}

Let $x$ denote a context (e.g.\ a state or prompt). For each context, we sample $K$ candidates $y_1,\dots,y_K \sim \pi_{\text{old}}(\cdot \mid x)$ and score them with a scalar scorer $S$. In our on-policy experiments, $\pi_{\text{old}}$ is simply the rollout-time snapshot of the current policy. We standardize the raw scores $s_i = S(x, y_i)$ within each group to obtain $u_i = [\operatorname{standardize}(s)]_i$, mapping the zero-variance case to $u = 0$ (Appendix~\ref{app:whitening}).

Let $\ell_i^\theta = \log \pi_\theta(y_i \mid x)$ denote the log-probability the current policy assigns to candidate $i$. The policy over the group is
\begin{equation}
\label{eq:restricted}
p_i^\theta
=
\frac{\exp(\ell_i^\theta)}
{\sum_{j=1}^K \exp(\ell_j^\theta)}.
\end{equation}
Writing $p_i^{\text{old}}$ for the same quantity under $\pi_{\text{old}}$, frozen at rollout time, we tilt this distribution toward higher-scoring candidates to form the target
\begin{equation}
\label{eq:target}
q_i
=
\frac{p_i^{\text{old}} \exp(u_i / \eta)}
{\sum_{j=1}^K p_j^{\text{old}} \exp(u_j / \eta)},
\end{equation}
where $\eta > 0$ is a temperature (we use $\eta = 1$ throughout; Appendix~\ref{app:temperature} shows this is robust).

We fit the policy to this target by minimizing the cross-entropy
\begin{equation}
\label{eq:loss}
\mathcal{L}_{\text{TPO}}(\theta)
=
-\sum_{i=1}^K q_i \log p_i^\theta,
\end{equation}
treating $q$ as fixed. The loss gradient satisfies $\partial \mathcal{L}/\partial \ell_i^\theta = p_i^\theta - q_i$, so gradient descent moves in direction $q_i - p_i^\theta$ and vanishes once the policy matches the target.

In the on-policy setting, the full update takes a few lines of code (Figure~\ref{fig:code}). If rollouts are reused for additional optimization epochs, $q$ stays frozen while the $\log p$ term is recomputed under $\theta$.

\paragraph{Why standardize.} The target (Eq.~\ref{eq:target}) exponentiates the scores, so groups with the same ranking but different numerical spread would produce very different targets. For example, $(1, 0, -1)$ and $(100, 0, -100)$ express the same ordering, but exponentiating $(100, 0, -100)$ makes the target nearly deterministic while $(1, 0, -1)$ yields a gentle tilt. Standardization makes the update depend on relative within-group performance rather than arbitrary score units, and largely removes the need to tune $\eta$.

\paragraph{KL-regularized interpretation.} The target $q$ is equivalently the unique solution of
\begin{equation}
\label{eq:regularized-improvement}
q
=
\arg\max_{r \in \Delta^{K-1}}
\left\{
\sum_{i=1}^K r_i u_i
- \eta\,\mathrm{KL}(r \,\|\, p^{\text{old}})
\right\},
\end{equation}
where $\Delta^{K-1}$ is the simplex over the sampled candidates.

\begin{proposition}
\label{prop:tpo-characterization}
Assume $p_i^{\text{old}} > 0$ for every sampled candidate. Then the target in Eq.~\ref{eq:target} is the unique maximizer of Eq.~\ref{eq:regularized-improvement}. Furthermore, treating $q$ as fixed, the cross-entropy loss in Eq.~\ref{eq:loss} satisfies $\nabla_{\ell^\theta}\mathcal{L}_{\text{TPO}} = p^\theta - q$, so the unique stationary distribution over the sampled candidates is $p^\theta = q$.
\end{proposition}

\begin{proof}
The objective in Eq.~\ref{eq:regularized-improvement} is strictly concave in $r$ because $-\mathrm{KL}(r \,\|\, p^{\text{old}})$ is strictly concave on the simplex when $p^{\text{old}}$ has full support. Introducing a Lagrange multiplier for $\sum_i r_i = 1$ and differentiating gives
\[
u_i - \eta\Bigl(\log \frac{r_i}{p_i^{\text{old}}} + 1\Bigr) + \lambda = 0,
\]
hence $r_i = C\,p_i^{\text{old}}\exp(u_i/\eta)$ for a normalization constant $C$, which yields Eq.~\ref{eq:target}.

Treating $q$ as fixed, differentiating the softmax cross-entropy with respect to the group logits gives $\partial \mathcal{L}/\partial \ell_i^\theta = p_i^\theta - q_i$. Therefore $\nabla_{\ell^\theta}\mathcal{L}_{\text{TPO}} = 0$ iff $p^\theta = q$, which identifies the unique stationary distribution over the sampled candidates.
\end{proof}

\begin{algorithm}[t]
\caption{Target Policy Optimization (TPO)}
\label{alg:tpo}
\renewcommand{\algorithmicrequire}{\textbf{Input:}}
\begin{algorithmic}[1]
\REQUIRE Policy $\pi_\theta$, scorer $S$, candidates per context $K$, temperature $\eta$ (default 1).
\REPEAT
    \STATE Freeze the behavior policy: $\pi_{\text{old}} \leftarrow \pi_\theta$.
    \STATE Sample a batch of contexts $x$ and candidates $\{y_i\}_{i=1}^K \sim \pi_{\text{old}}(\cdot \mid x)$.
    \STATE Compute scores $s_i = S(x, y_i)$ and form $s=(s_1,\dots,s_K)$.
    \STATE Standardize: $u_i = [\operatorname{standardize}(s)]_i$.
    \STATE Compute the target
    \[
    q_i
    =
    \frac{p_i^{\text{old}} \exp(u_i / \eta)}
    {\sum_{j=1}^K p_j^{\text{old}} \exp(u_j / \eta)}.
    \]
    \STATE Take one or more gradient steps on
    \[
    \mathcal{L}_{\text{TPO}}(\theta) = -\sum_{i=1}^K q_i \log p_i^\theta,
    \]
    treating $q$ as fixed.
\UNTIL{converged}
\end{algorithmic}
\end{algorithm}

\begin{figure*}[t]
\begin{minipage}[t]{0.48\textwidth}
\centering{\small\bfseries JAX}\vspace{0.5em}
\begin{lstlisting}[style=code]
def tpo_target(log_scores, u, eta=1.0):
    return jax.nn.softmax(
        jax.nn.log_softmax(log_scores, -1)
        + u / eta, -1)

q = jax.lax.stop_gradient(
    tpo_target(log_scores, u))
log_p = jax.nn.log_softmax(log_scores, -1)
loss = -(q * log_p).sum(-1).mean()
\end{lstlisting}
\end{minipage}%
\hfill
\begin{minipage}[t]{0.48\textwidth}
\centering{\small\bfseries PyTorch}\vspace{0.5em}
\begin{lstlisting}[style=code]
def tpo_target(log_scores, u, eta=1.0):
    return F.softmax(
        F.log_softmax(log_scores, -1)
        + u / eta, -1)

q = tpo_target(
    log_scores, u).detach()
log_p = F.log_softmax(log_scores, -1)
loss = -(q * log_p).sum(-1).mean()
\end{lstlisting}
\end{minipage}
\caption{\textbf{Implementation sketch.}
\texttt{log\_scores} contains the policy log-probabilities of the sampled candidates, renormalized by \texttt{log\_softmax} to form the policy over the group; \texttt{u} contains standardized task scores; \texttt{eta} is an optional temperature with default value 1. The sketch shows the simplest on-policy implementation, where the same \texttt{log\_scores} tensor is used both to form $q$ and to compute \texttt{log\_p}, with $q$ detached from the computation graph before the update.}
\label{fig:code}
\end{figure*}

\section{Experiments}
\label{sec:experiments}

Baselines are PPO~\citep{schulman2017ppo}, GRPO~\citep{shao2024deepseekmath}, and DG~\citep{osband2026dg}. For dense-reward experiments, we compare token-level grouped variants (TPO\textsubscript{token}, GRPO\textsubscript{token}) that sample $K{=}8$ next-token candidates at each prefix state; for terminal reward, we use sequence-level TPO and GRPO with $K{=}8$ full rollouts per prompt; for LLM RLVR, $K{=}16$. PPO, GRPO, and TPO take multiple gradient epochs per rollout batch; DG uses a single epoch, following \citet{osband2026dg}, because it diverges with more (Appendix~\ref{app:dg-multiepoch}). Our GRPO baseline uses the clipped surrogate with $z$-scored advantages and a reverse-KL penalty (Appendix~\ref{app:grpo-baseline}); we refer to it simply as GRPO throughout.

Where grouped methods consume $K\times$ more rollouts than single-sample methods for the same number of prompts, we report two comparisons. \emph{Prompt-matched}: same number of prompts; grouped methods use more total rollouts. \emph{Interaction-matched}: same total rollouts; single-sample methods see more prompts.

Unless stated otherwise, all transformer experiments use Optax's \citep{deepmind2020jax} Muon optimizer~\citep{jordan2024muon} at learning rate $10^{-3}$ and batch size $B{=}100$, with Muon applied to 2D parameter tensors and AdamW to non-2D tensors.

\subsection{Single-context bandit: within-context update quality}
\label{sec:exp-tabular-single}

Following \citet{osband2026dg}, we replace the network with explicit logit tables so the softmax policy and its gradients can be computed exactly. These tabular runs do not use a neural optimizer; we take normalized logit steps of size $\alpha{=}0.1$ directly.

We consider a $K$-armed bandit with one correct action $y^*$ among $K{=}100$ choices. The reward is $R = \mathbf{1}\{A = y^*\}$. At each step, the agent samples $B{=}100$ actions, computes a gradient estimate, and takes a normalized step. We average over 100 seeds.

\begin{figure}[t]
\centering
\includegraphics[width=\textwidth]{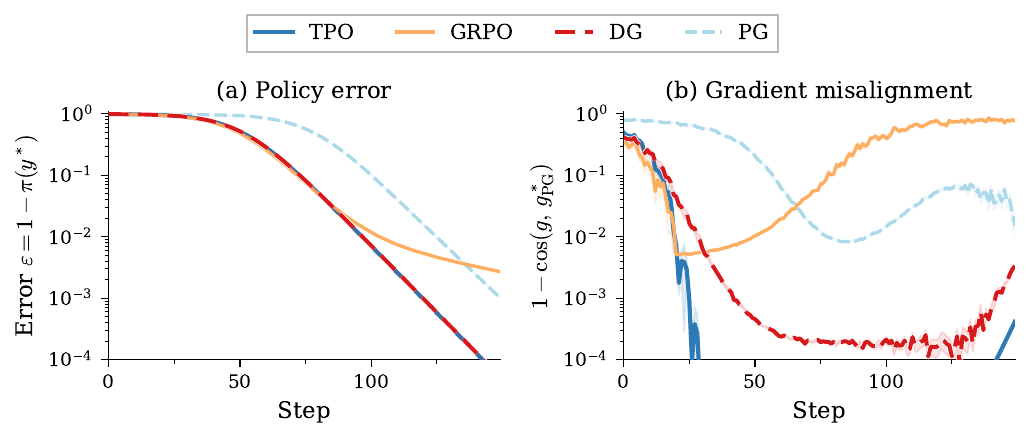}
\caption{\textbf{Single-context symmetric bandit} ($K{=}100$, $B{=}100$, normalized steps). (a)~TPO and DG converge fastest; GRPO and PG plateau at higher error. (b)~TPO maintains the lowest misalignment to the oracle gradient throughout training.}
\label{fig:tabular_single}
\end{figure}

Figure~\ref{fig:tabular_single} shows that TPO and DG converge fastest. Unlike PG and GRPO, they continue improving beyond 1\% error. The misalignment panel shows why: TPO stays closest to the oracle policy-gradient direction as the policy concentrates, while GRPO becomes increasingly misaligned.

\subsection{Multi-context bandit: cross-context allocation}
\label{sec:exp-tabular-multi}

The single-context experiment tests update quality; this one tests how a normalized update allocates a finite step budget \emph{across} contexts. We consider $N{=}100$ independent contexts, each a $K{=}10$ bandit with $\mathcal{N}(0,1)$ logit initialization~\citep{osband2026dg}. Exact population updates remove sampling variance, so any remaining gap reflects how each method distributes the step. We include the cross-entropy (CE) oracle, which is optimal under normalized steps in this setting.

\begin{figure}[t]
\centering
\includegraphics[width=\textwidth]{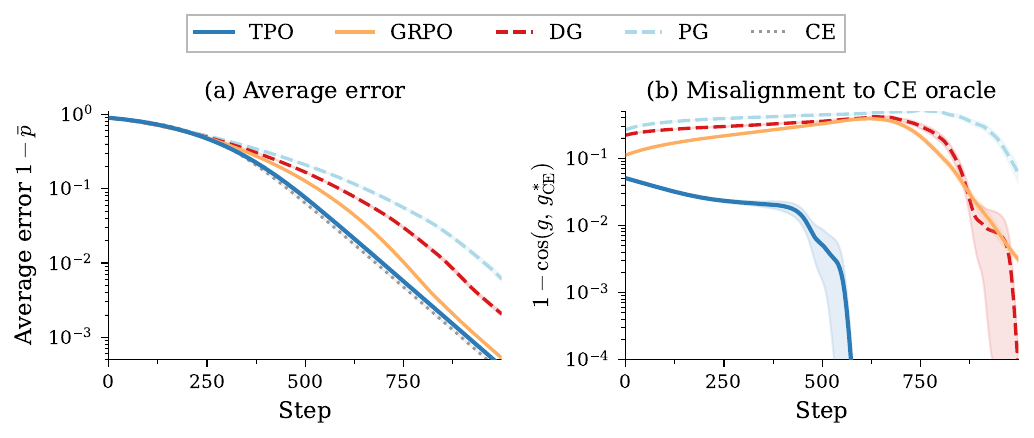}
\caption{\textbf{Multi-context bandit} ($N{=}100$, $K{=}10$, exact gradients). (a)~All methods converge; the CE oracle is fastest. (b)~TPO achieves near-zero misalignment to the CE oracle direction, confirming its update direction targets the optimal allocation.}
\label{fig:tabular_multi}
\end{figure}

Figure~\ref{fig:tabular_multi} shows that all methods eventually converge and that CE is fastest, but among the RL updates TPO is the closest to CE in both error and direction. DG and GRPO improve slightly faster at the start, but TPO overtakes them after the early transient and finishes with the lowest error of the three. The misalignment panel shows the same pattern more clearly: TPO remains much closer to the CE direction throughout training.

This pattern is analytically transparent in the one-hot setting. Let $p_n=\pi_n(y_n)$ be the current probability of the correct action in context $n$. Working in \emph{logit space} with baseline $b{=}0$, every exact update can be written as
\[
g_n = \beta(p_n)\,(e_{y_n}-\pi_n),
\]
so all methods share the same within-context direction $e_{y_n}-\pi_n$ and differ only in the scalar weight $\beta(p_n)$. Because the global step is normalized, $\beta$ controls how much of that step is spent on context~$n$: a method that assigns larger $\beta$ to easy (high-$p_n$) contexts wastes budget where it is least needed.

The coefficients (derived in Appendix~\ref{app:tabular-multi-weight}) are:
\[
\beta_{\mathrm{CE}}(p_n)=1,\qquad
\beta_{\mathrm{DG}}(p_n)=\frac{p_n}{1+p_n},\qquad
\beta_{\mathrm{GRPO}}(p_n)=\sqrt{\frac{p_n}{1-p_n}},\qquad
\beta_{\mathrm{TPO}}(p_n)=\frac{p_n(\lambda-1)}{1-p_n+\lambda p_n},
\]
where $\lambda=\exp(u_{y_n}-u_{a\neq y_n})\approx 28$ for $K{=}10$.

CE treats every context equally ($\beta{=}1$ everywhere). DG and GRPO both \emph{vanish} as $p_n \to 0$: when a context is hard, they barely update it. DG vanishes linearly ($\beta \approx p_n$) and GRPO vanishes as $\sqrt{p_n}$, so both spend most of the normalized step on contexts that are already nearly solved. TPO's coefficient, by contrast, stays large even at small $p_n$: at $p_n{=}0.1$, $\beta_{\mathrm{TPO}}{=}0.73$ versus $0.09$ for DG and $0.33$ for GRPO. TPO therefore allocates more update budget to hard contexts, which is why it tracks the CE oracle more closely and overtakes the scalar-weighted baselines after the initial transient.

\subsection{Neural policy learning: MNIST contextual bandit}
\label{sec:exp-discrete}

Following \citet{osband2026dg}, we cast MNIST classification as a one-step contextual bandit: the agent samples $A \in \{0,\dots,9\}$ and receives $R = \mathbf{1}\{A = Y\}$ without observing the label $Y$. A two-layer ReLU network trains for 10{,}000 steps (20 seeds). Each method samples a single action per context and updates from bandit feedback alone.
We optimize the network with Adam at learning rate $10^{-3}$ and batch size $B{=}100$.

\begin{figure}[t]
\centering
\includegraphics[width=\textwidth]{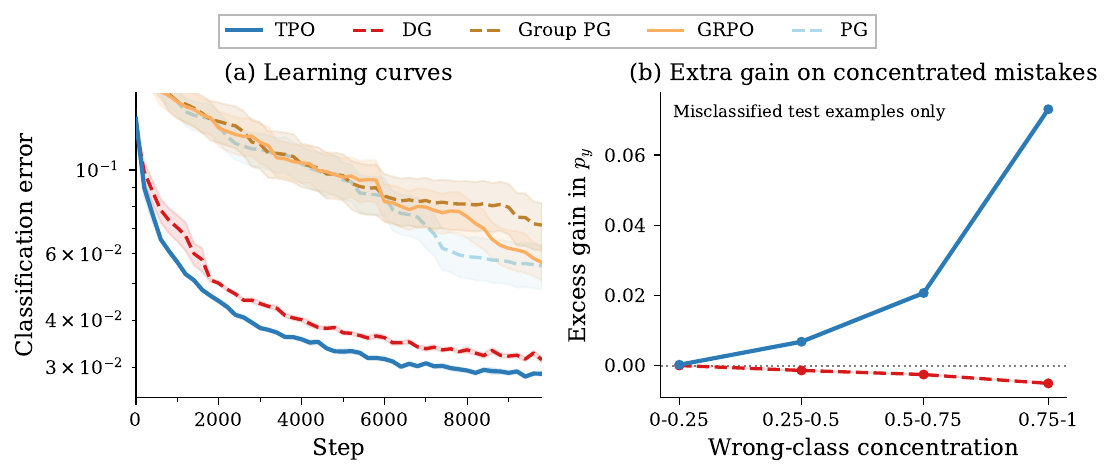}
\caption{\textbf{MNIST contextual bandit: TPO converges fastest and reaches the lowest error.} (a)~Learning curves for all single-sample bandit updates, including the same-signal ablation Group PG. (b)~At step 2{,}000, for each misclassified example we measure how much more each method increases the true-class probability $p_y$ compared to a generic one-vs-rest baseline (Appendix~\ref{app:mnist-logit}), binned by wrong-class concentration $c = \max_{j \ne y}\pi_j/(1-\pi_y)$. TPO's extra gain grows with concentration; DG's does not.}
\label{fig:mnist}
\end{figure}

Figure~\ref{fig:mnist} shows that the tabular pattern survives the transition to a neural policy: TPO converges fastest (5\% error at step 1{,}600 vs.\ 2{,}200 for DG) and reaches the lowest final error (2.9\%). With a single sampled action per context, GRPO reduces to batch-normalized REINFORCE and therefore performs comparably to PG (5.9\% vs.\ 5.3\%).

PG, single-sample GRPO, and Group PG all learn ``increase the true class versus the rest'' without using which wrong class was sampled --- in expectation, they collapse to a rescaled one-vs-rest direction $c(x)(e_y-\pi)$ (Appendix~\ref{app:mnist-logit}). DG and TPO both condition on the sampled action, but only TPO turns a failed sample into a class-specific target update: a correct sample pulls probability toward the label, while an incorrect sample directly suppresses the sampled wrong class. This extra structure should matter most when error mass is concentrated on one or a few confusing alternatives. Removing it confirms this: Group PG keeps the same candidates and standardized scores but replaces target matching with scalar-weighted REINFORCE, raising final error from 2.9\% to 7.2\%.

Figure~\ref{fig:mnist}(b) tests that prediction directly. On each misclassified test example, let $c = \max_{j \ne y}\pi_j/(1-\pi_y)$ denote the fraction of wrong-class mass carried by the most likely wrong label.

We then compare the exact first-order gain in $p_y$ to the scalar one-vs-rest surrogate from Appendix~\ref{app:mnist-logit}. TPO's surplus is near zero when the error mass is diffuse, but rises to $0.073$ in the highest-concentration bin at the step-2{,}000 checkpoint; DG stays slightly negative throughout. TPO's benefit therefore appears exactly where one-vs-rest corrections are too coarse: examples dominated by one confusing wrong label.

\subsection{Dense sequence reward: token-level transformer grouping}
\label{sec:exp-sequence}

Dense per-token rewards let us group at the token level. We use the Token Reversal task of \citet{osband2026dg}: a 2-layer, 4-head causal transformer autoregressively reverses an input sequence of length $H{=}10$ drawn uniformly from a vocabulary of size $V$. The reward is the \emph{bag-of-tokens} fraction of tokens reversed correctly. We sweep $V \in \{2, 4, 8, 16\}$, growing the output space from $2^{10} \approx 10^3$ to $16^{10} \approx 10^{12}$, and report \emph{sequence error} (fraction of tokens incorrect) averaged over 20 seeds.

At each prefix state, we sample $K{=}8$ next-token candidates and form the group over those candidates (TPO\textsubscript{token}, GRPO\textsubscript{token}). For autoregressive models, $\ell_i^\theta$ is the usual sum of per-token log-probabilities. All methods follow one behavior trajectory per prompt, so environment interactions are matched.

\begin{figure}[t]
\centering
\includegraphics[width=\textwidth]{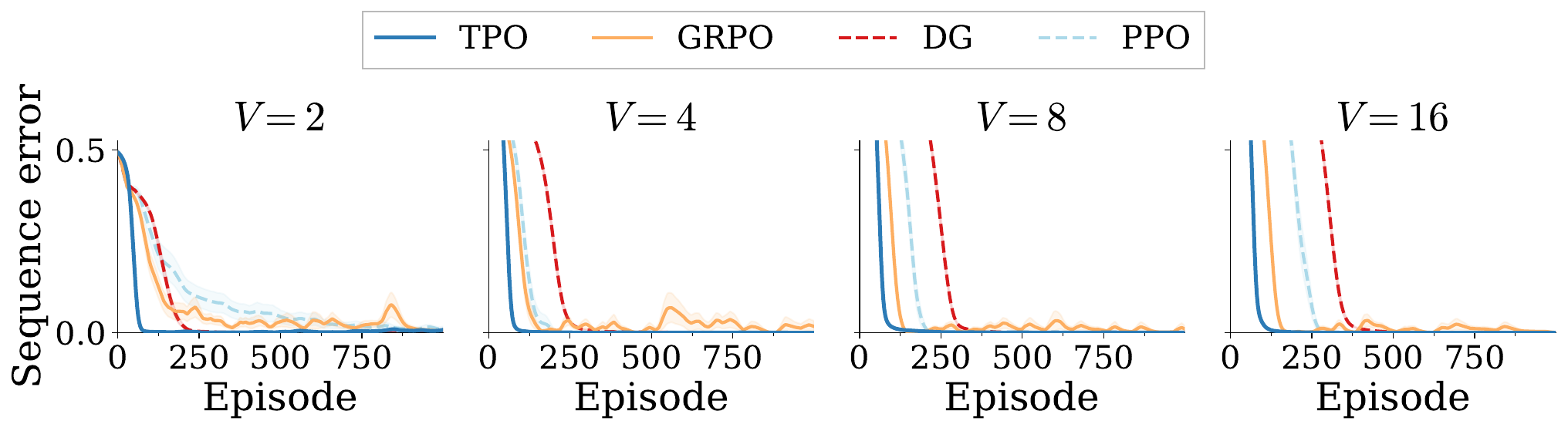}
\caption{\textbf{Token Reversal} (bag-of-tokens reward, $K{=}8$ token candidates). All methods use $B{=}100$ prompts and follow one behavior trajectory each; TPO\textsubscript{token} and GRPO\textsubscript{token} additionally sample $K$ next-token candidates at each prefix state. Columns vary vocabulary size $V \in \{2, 4, 8, 16\}$.}
\label{fig:transformer}
\end{figure}

The gap between methods widens with task difficulty (Table~\ref{tab:transformer}, Figure~\ref{fig:transformer}): at $V{=}16$, TPO\textsubscript{token} reaches 1\% error at step 102, compared to 148 for GRPO\textsubscript{token}, 259 for PPO, and 393 for DG.

\begin{table}[t]
  \caption{\textbf{Steps to 1\% error.} Token Reversal (bag-of-tokens reward, $K{=}8$ token candidates). \textbf{Bold}: fastest method at each $V$. All methods use the same environment interactions per step.}
  \label{tab:transformer}
  \centering
  \small
  \begin{tabular}{lcccc}
    \toprule
    & $V=2$ & $V=4$ & $V=8$ & $V=16$ \\
    \midrule
    TPO\textsubscript{token}  & \textbf{58}  & \textbf{74}  & \textbf{103} & \textbf{102} \\
    GRPO\textsubscript{token} & 904          & 141          & 124          & 148 \\
    DG                        & 199          & 273          & 314          & 393 \\
    PPO                       & 872          & 181          & 191          & 259 \\
    \bottomrule
  \end{tabular}
\end{table}

Because all methods follow a single behavior trajectory per prompt, there is no prompt-matched vs.\ interaction-matched distinction, rollout budgets are identical. GRPO\textsubscript{token} improves with larger $V$ (where more token candidates provide a richer signal) but lags behind TPO\textsubscript{token} throughout. DG and PPO, which lack within-group structure, scale less favorably.

\subsection{Generalization across task and reward variants}
\label{sec:exp-variations}

Does the pattern hold beyond token reversal? Following \citet{osband2026dg}, we evaluate four target logics (copy, flip, reverse copy, reverse flip) under two reward structures (bag-of-tokens and sequential), yielding eight variants. Sequential reward gives credit only up to the first incorrect token, sparser than bag-of-tokens but denser than terminal. Hyperparameters match Section~\ref{sec:exp-sequence} ($H{=}10$, $V{=}2$, $K{=}8$ token candidates); 10 seeds, 1{,}000 episodes.

\begin{figure}[t]
\centering
\includegraphics[width=\textwidth]{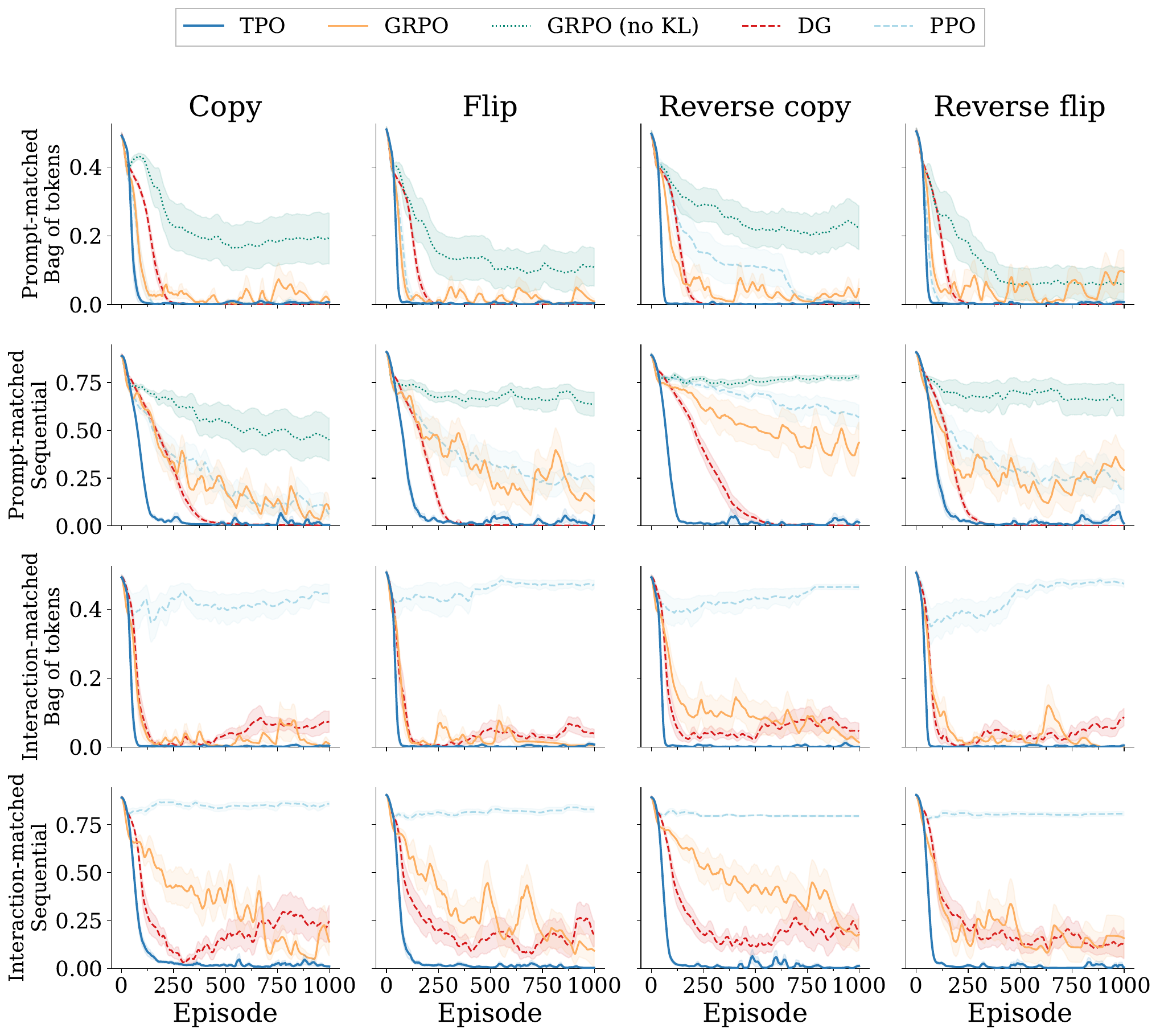}
\caption{\textbf{Task variations, prompt- and interaction-matched.} \emph{Top two rows}: prompt-matched. \emph{Bottom two rows}: interaction-matched. Within each pair, the first row is bag-of-tokens reward and the second is sequential reward. Columns vary target logic.}
\label{fig:variations}
\end{figure}

Under bag-of-tokens reward (top row of Figure~\ref{fig:variations}), TPO\textsubscript{token} reaches 1\% error first on all eight variants (Table~\ref{tab:variations}), 2--6$\times$ faster than the runner-up. All methods except PPO eventually reach 1\% on bag-of-tokens tasks. Under sequential reward, TPO\textsubscript{token}'s advantage widens: it reaches 1\% error on all four tasks within our budget; DG converges on all four but more slowly; GRPO\textsubscript{token} and PPO fail to converge on any.

\begin{table}[t]
  \caption{\textbf{Steps to 1\% error, task variations} ($K{=}8$ token candidates). \textbf{Bold}: fastest per row. ``$-$'': never reached within budget.}
  \label{tab:variations}
  \centering
  \small
  \begin{tabular}{llcccc}
    \toprule
    Reward & Target & TPO\textsubscript{token} & GRPO\textsubscript{token} & DG & PPO \\
    \midrule
    \multirow{4}{*}{Bag of tokens}
    & Copy         & \textbf{81}  & 338  & 219  & 170 \\
    & Flip         & \textbf{56}  & 104  & 201  & 146 \\
    & Rev.\ copy   & \textbf{55}  & 352  & 202  & $-$ \\
    & Rev.\ flip   & \textbf{59}  & 209  & 200  & 143 \\
    \midrule
    \multirow{4}{*}{Sequential}
    & Copy         & \textbf{295} & $-$  & 439  & $-$ \\
    & Flip         & \textbf{321} & $-$  & 349  & $-$ \\
    & Rev.\ copy   & \textbf{159} & $-$  & 515  & $-$ \\
    & Rev.\ flip   & \textbf{276} & $-$  & 309  & $-$ \\
    \bottomrule
  \end{tabular}
\end{table}

Under sequential reward, only TPO\textsubscript{token} and DG converge. The key is per-state targeting: under sequential reward, prefixes after the first mistake see zero reward for every candidate, so the target there matches the old policy and introduces no spurious signal. TPO\textsubscript{token} therefore concentrates its update on informative prefixes where at least one candidate continues correctly. DG's sigmoid gating also helps but is slower; GRPO\textsubscript{token} and PPO lack an equally explicit local target.

\subsection{Sparse credit assignment: terminal reward}
\label{sec:exp-rlvr}

The hardest credit-assignment test removes intermediate feedback entirely: the model receives an exact-match reward only after the full sequence. Without per-token rewards, we revert to sequence-level TPO and GRPO, each sampling $K{=}8$ complete rollouts per prompt. Prompt-matched runs use $B{=}100$; interaction-matched runs scale single-sample batch size and learning rate by $K$ and $\sqrt{K}$ respectively. Other hyperparameters match Section~\ref{sec:exp-sequence} ($V{=}2$); we sweep $H \in \{7, 8, 9, 10\}$ over 2{,}000 episodes. We report exact-match error (fraction of sequences with any mistake), not token-level error.

\begin{figure}[t]
\centering
\includegraphics[width=\textwidth]{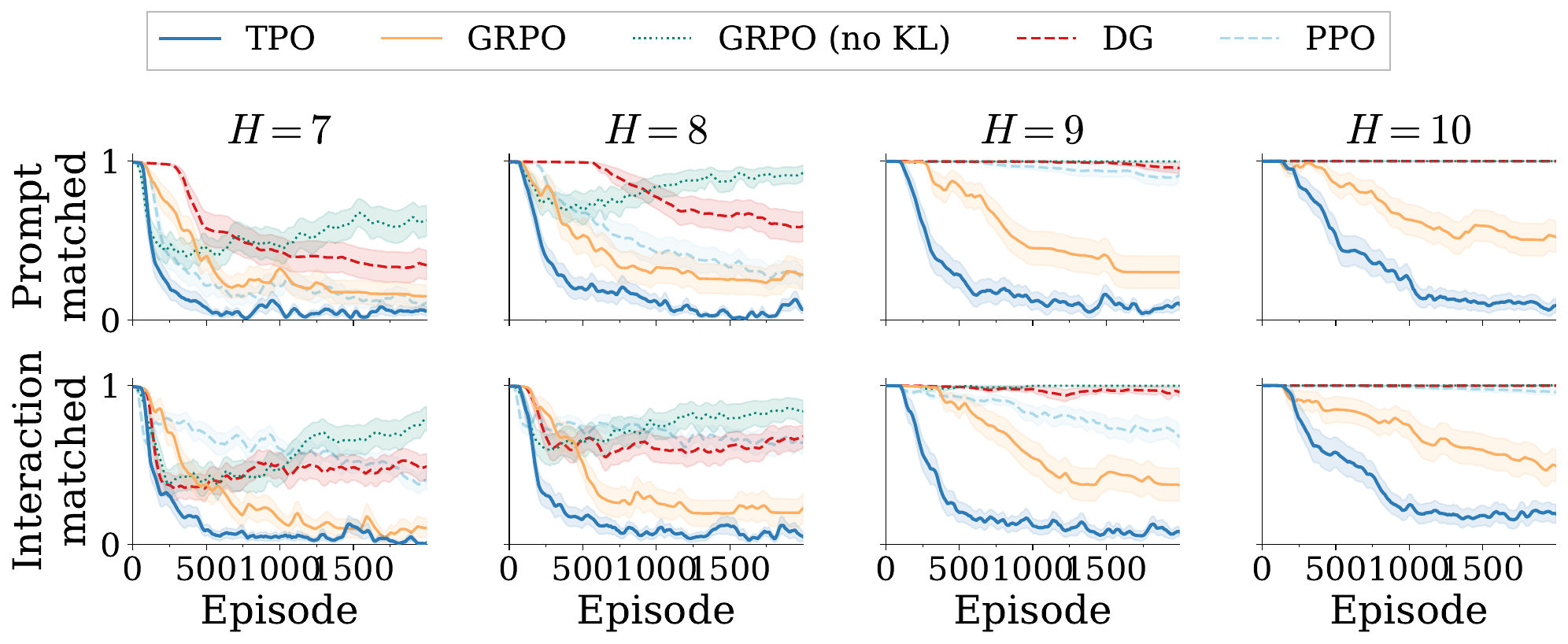}
\caption{\textbf{Terminal reward, prompt- and interaction-matched.} \emph{Top row}: prompt-matched ($B{=}100$ for all methods). \emph{Bottom row}: interaction-matched ($B{\cdot}K{=}800$ rollouts per step, with single-sample batch size and learning rate scaled by $K$ and $\sqrt{K}$ respectively). Here grouped methods use $K{=}8$ candidates per prompt. Y-axis: exact-match error. TPO has the lowest error at each $H$ under both matching conditions.}
\label{fig:rlvr}
\end{figure}

Under prompt matching, the methods diverge most (Table~\ref{tab:rlvr}, top row of Figure~\ref{fig:rlvr}): TPO attains the lowest error at each tested $H$. GRPO and PPO make progress at shorter lengths but degrade steeply; DG fails earlier still. Removing GRPO's KL penalty ($\beta{=}0$) makes it substantially worse (66.6\% at $H{=}7$ and no meaningful learning beyond $H{=}8$), showing that the KL term is GRPO's primary stabilizer under sparse reward.

\begin{table}[t]
  \caption{\textbf{Exact-match error (\%), terminal reward.} \textbf{Bold}: best method. ``$-$'': $>$95\% (no meaningful learning). Left: prompt-matched. Right: interaction-matched. TPO attains the lowest error at each tested $H$.}
  \label{tab:rlvr}
  \centering
  \small
  \begin{tabular}{lcccc|cccc}
    \toprule
    & \multicolumn{4}{c|}{Prompt-matched} & \multicolumn{4}{c}{Interaction-matched} \\
    & $H{=}7$ & $H{=}8$ & $H{=}9$ & $H{=}10$ & $H{=}7$ & $H{=}8$ & $H{=}9$ & $H{=}10$ \\
    \midrule
    TPO  & \textbf{6.9}  & \textbf{8.6}  & \textbf{6.1}  & \textbf{7.4}  & \textbf{1.8}  & \textbf{2.8}  & \textbf{5.3}  & \textbf{19.0} \\
    GRPO & 14.5           & 27.6           & 30.0           & 50.4           & 9.6            & 23.2           & 36.2           & 48.7 \\
    GRPO (no KL) & 66.6   & 92.5           & $-$            & $-$            & 78.1           & 83.8           & $-$            & $-$ \\
    PPO  & 12.0           & 26.3           & 90.6           & $-$            & 38.6           & 62.1           & 66.2           & $-$ \\
    DG   & 33.8           & 58.8           & $-$            & $-$            & 47.7           & 69.4           & $-$            & $-$ \\
    \bottomrule
  \end{tabular}
\end{table}

Under interaction matching (bottom row of Figure~\ref{fig:rlvr}, right half of Table~\ref{tab:rlvr}), TPO remains ahead at each $H$. The gap is wider here than in the bag-of-tokens experiments, where interaction matching narrowed it substantially. With terminal reward, the bottleneck is not gradient variance but extracting useful signal from sparse outcomes, the regime where target matching matters most.

\subsection{Anchor and target-matching ablations}
\label{sec:exp-anchor-group-pg}

To isolate ingredients of TPO's grouped update, we compare TPO against several prompt-matched variants on the same terminal-reward benchmark ($H \in \{7,8,10\}$, $V{=}2$, $K{=}8$, $B{=}100$, 20 seeds). All methods use the same grouped full-sequence rollouts. ``TPO-no-anchor'' removes the $p^{\text{old}}$ anchor ($q_i \propto \exp(u_i)$). ``Group PG'' keeps the same candidates and standardized scores but replaces target matching with scalar-weighted policy gradient. ``GRPO (no KL)'' removes the reverse-KL penalty ($\beta{=}0$).

\begin{figure}[t]
\centering
\includegraphics[width=\textwidth]{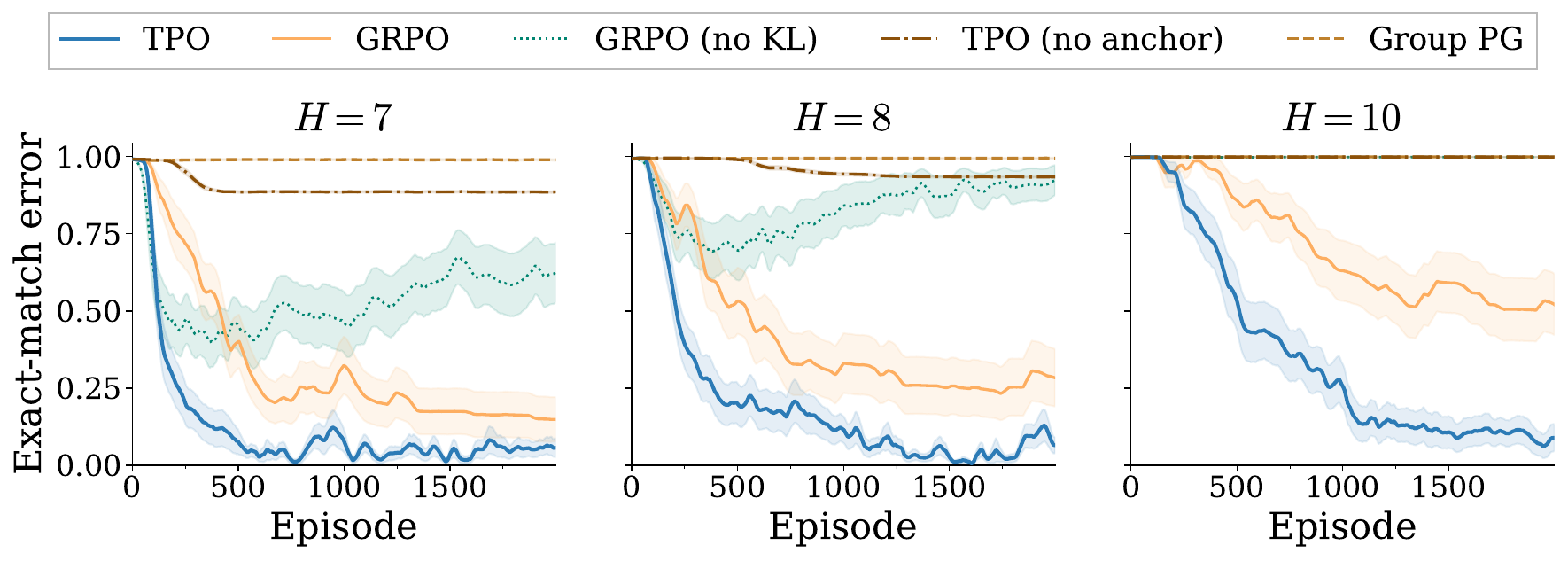}
\caption{\textbf{Removing the anchor, KL penalty, or target matching each degrades learning.} Terminal reward, reverse-copy targets, $V{=}2$, $K{=}8$, $B{=}100$, 20 seeds. Shading shows $\pm 1$ s.e.}
\label{fig:anchor-group-pg}
\end{figure}

Full TPO outperforms every ablation at each sequence length (Figure~\ref{fig:anchor-group-pg}), and the gaps widen with $H$: at $H{=}10$, TPO reaches 7.4\% while every ablation exceeds 99\%. The old-policy anchor is doing real work: removing it is consistently harmful. Target matching itself also matters: keeping the same candidates and standardized scores but reverting to scalar weighting (Group PG) performs worst. Removing GRPO's KL penalty makes it substantially worse, consistent with Section~\ref{sec:exp-rlvr}.

\subsection{LLM RLVR: transfer to billion-parameter models}
\label{sec:exp-llm-rlvr}

GRPO is the de facto standard for billion-parameter LLM RLVR~\citep{lambert2025tulu3pushingfrontiers,guo2025deepseekr1}. Does TPO's advantage transfer to this setting?

We compare TPO and GRPO using the \texttt{verl} stack~\citep{sheng2024verl} on two models (Qwen3-1.7B~\citep{yang2025qwen3} and DeepSeek-R1-Distill-Qwen-1.5B~\citep{guo2025deepseekr1}) and three tasks: GSM8K~\citep{cobbe2021gsm8k}, graph coloring, and Knights \& Knaves (both from Reasoning Gym~\citep{stojanovski2025reasoninggym}). All runs use $K{=}16$ rollouts per prompt; the paired runs differ only in the policy loss (TPO vs.\ clipped surrogate with $z$-scored advantages). Implementation details (optimizer, LoRA, hardware) are in Appendix~\ref{app:llm-rlvr-details}.

\begin{figure}[t]
\centering
\includegraphics[width=\textwidth]{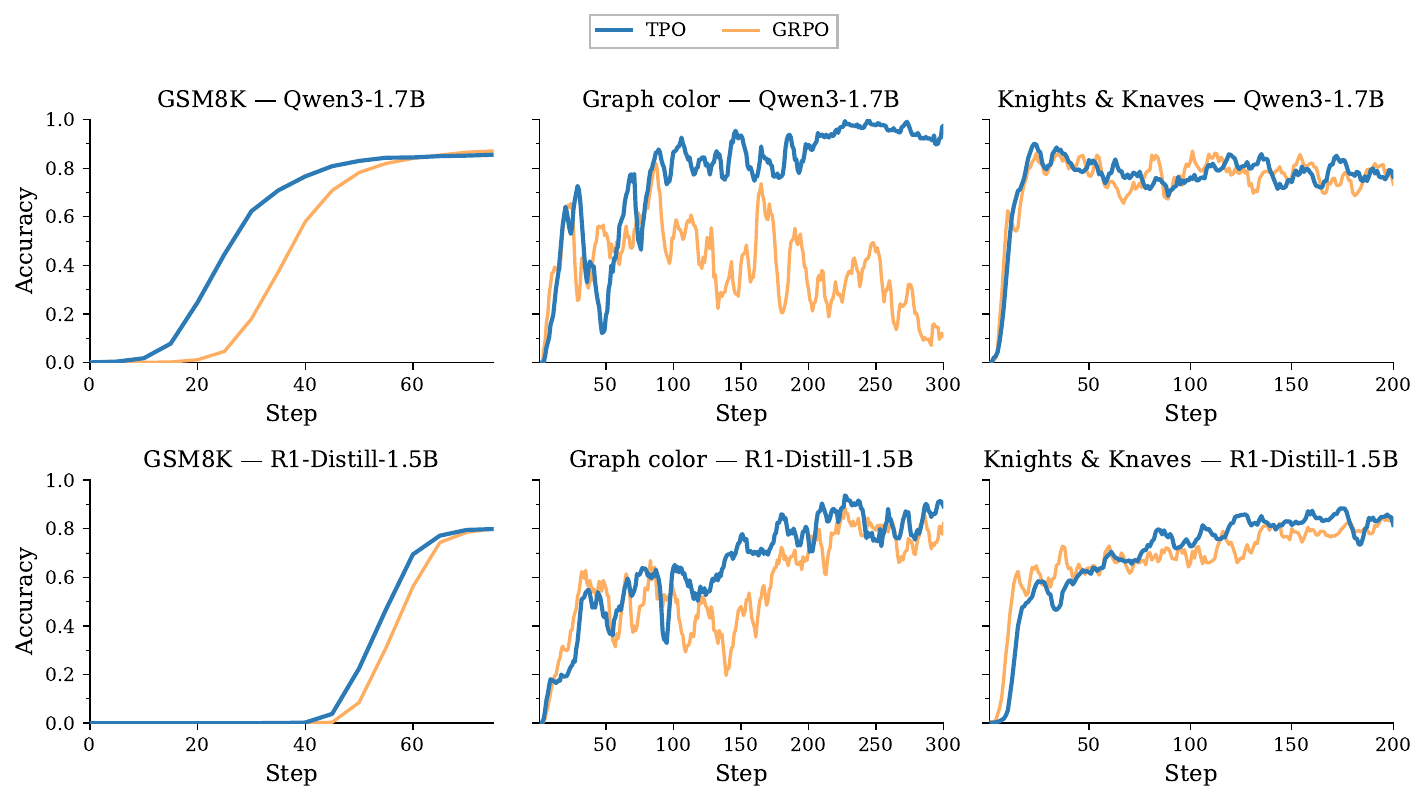}
\caption{\textbf{LLM RLVR.} Top row: Qwen3-1.7B. Bottom row: DeepSeek-R1-Distill-Qwen-1.5B. All runs use $K{=}16$ rollouts per prompt. Columns: GSM8K (held-out test accuracy, evaluated every 5 steps), Reasoning Gym graph coloring (train mean score), Reasoning Gym Knights \& Knaves (train mean score).}
\label{fig:llm_rlvr}
\end{figure}

On GSM8K (Figure~\ref{fig:llm_rlvr}, left column), TPO learns faster early (reaching 50\% accuracy ${\sim}$10 steps before GRPO on Qwen3-1.7B) but both converge to comparable final accuracy (${\sim}$85--87\%), consistent with TPO's advantage being largest during the learning phase.

On the Reasoning Gym tasks (middle and right columns), we plot train mean score. The gap is starker here: on graph coloring, GRPO fails entirely on Qwen3-1.7B (near-zero score for 300 steps) while TPO reaches ${\sim}$0.96. On R1-Distill-1.5B, both learn but TPO converges higher (${\sim}$0.96 vs.\ ${\sim}$0.81). Knights \& Knaves shows the same pattern. These harder tasks expose TPO's advantage more clearly than GSM8K, where both methods eventually saturate.

\section{What explains TPO's gains under sparse reward?}
\label{sec:analysis}

We identify several reinforcing properties: the gradient self-extinguishes once the policy matches the target, signal concentrates on the few informative groups rather than all-fail batches, and the fixed target supports stable multi-epoch reuse. We examine these in a representative sparse-reward regime ($H{=}8$, $V{=}2$, $K{=}32$, $B{=}256$, 2{,}000 episodes). We compute per-step diagnostics from the original 10-seed runs; the $K$-sweep, epoch sweep, and masking ablations use 30 seeds.

\subsection{Does TPO's gradient vanish in practice while GRPO's persists?}
\label{sec:analysis-gradient}

Because TPO's gradient vanishes at $p^\theta = q$ (Proposition~\ref{prop:tpo-characterization}), the update should decay as the policy converges. Policy-gradient methods lack this fixed point.

\begin{figure}[t]
\centering
\includegraphics[width=\textwidth]{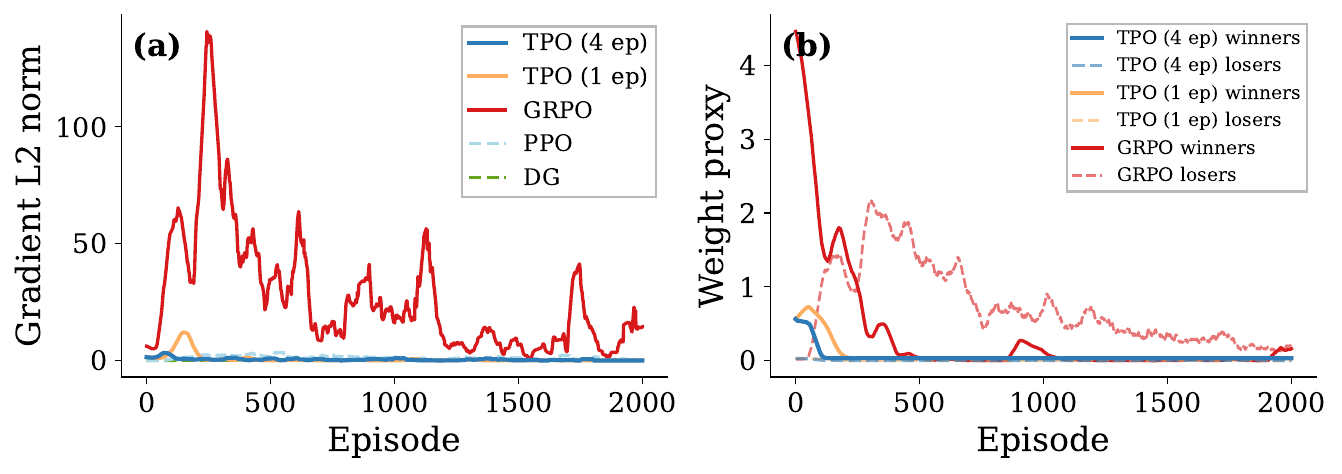}
\caption{\textbf{TPO's gradient self-extinguishes; GRPO's does not} ($H{=}8$, $V{=}2$, $K{=}32$). (a)~Gradient L2 norms over training. (b)~Per-candidate weight proxy on successful (solid) vs.\ failed (dashed) candidates: mean target mass $q_i$ for TPO, mean $|A_i|$ for GRPO.}
\label{fig:ablation-gradient}
\end{figure}

Figure~\ref{fig:ablation-gradient}(a) tracks the L2 norm of the first-epoch gradient over training. TPO's gradient spikes during the learning phase then decays to near zero once the policy converges (${\sim}$episode~300). GRPO maintains persistent gradient norms throughout training, even after its error curve plateaus at 12.7\%. GRPO's policy keeps moving even after its error plateaus, rather than settling near a fixed point.

Figure~\ref{fig:ablation-gradient}(b) shows a per-candidate weight proxy on successful candidates (solid) versus failed ones (dashed). Because the proxy differs between methods (target mass $q_i$ for TPO, advantage magnitude $|A_i|$ for GRPO), the panel is an allocation diagnostic, not a gradient decomposition. TPO rapidly removes weight from failed candidates, whereas GRPO continues assigning nonzero advantage magnitude to failures even late in training. The stronger fixed-point claim comes from panel~(a): TPO's gradient norm collapses near zero, while GRPO's does not.

\subsection{How does TPO allocate signal when informative groups are rare?}
\label{sec:analysis-zerovar}

With $K{=}32$ candidates and a per-sequence success rate of $(1/V)^H \approx 0.4\%$ at initialization, most groups contain no successful completion.

\begin{figure}[t]
\centering
\includegraphics[width=\textwidth]{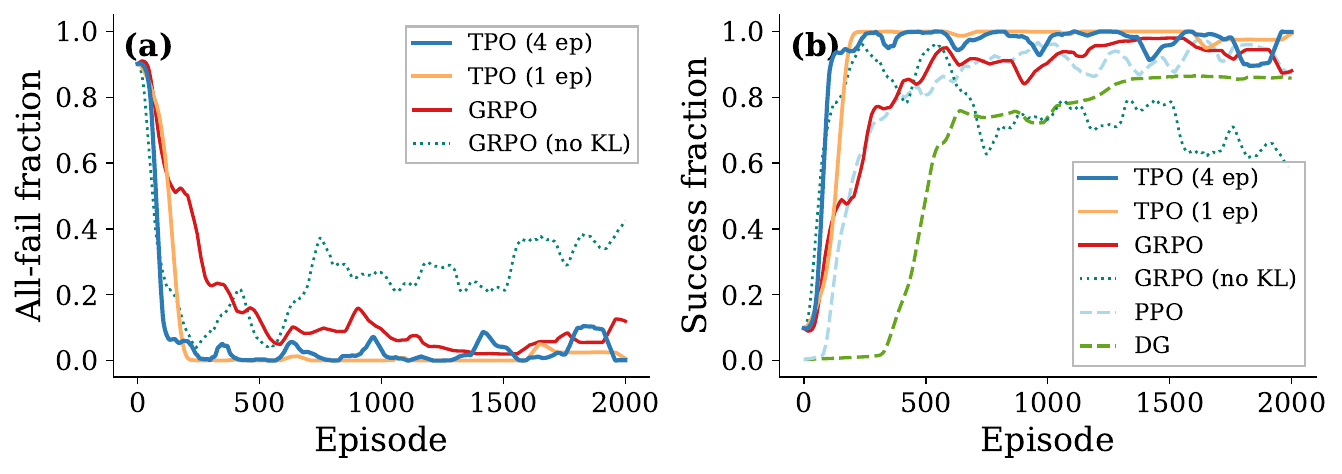}
\caption{\textbf{Most groups carry no signal early in training; TPO eliminates them fastest} ($H{=}8$, $V{=}2$, $K{=}32$). (a)~Fraction of groups where all $K$ candidates fail. (b)~Fraction of prompts with at least one successful candidate.}
\label{fig:ablation-zerovar}
\end{figure}

Figure~\ref{fig:ablation-zerovar}(a) shows that roughly 90\% of groups are all-fail at the start of training. For TPO, these groups are neutral on the rollout snapshot: zero score variance implies standardized scores $u=0$ (Appendix~\ref{app:whitening}), so $q = p^{\text{old}}$ and the grouped-loss contribution on the first epoch is exactly zero. Early in training, TPO therefore spends its first-epoch grouped update on the relatively small fraction of groups that actually distinguish better from worse candidates, namely those containing at least one success. (We show all-fail groups rather than total zero-variance groups because late in training zero variance can also arise from all-success groups, which are not the sparse-reward failure mode of interest.)

In the shared-parameter multi-epoch setting, that neutrality need not persist forever. Once informative groups have moved the policy away from the rollout snapshot, revisiting the same all-fail group yields an anchor term back toward $p^{\text{old}}$. This later-epoch pullback can arise for both TPO and our snapshot-KL GRPO baseline, so zero-variance groups are not permanently ignored. The key property is narrower: on the rollout snapshot, when informative groups are scarce, TPO concentrates its grouped signal on the groups that contain actual ranking information.

As training progresses and the policy improves, the all-fail fraction drops: more groups contain at least one successful candidate (Figure~\ref{fig:ablation-zerovar}(b)). This means a larger fraction of each batch contributes nontrivial target structure. TPO drives the all-fail fraction to near zero quickly, whereas GRPO leaves a larger residual and GRPO (no~KL) remains substantially worse.

\paragraph{Group-size ablation.}
Varying $K \in \{4,8,16,32,64\}$ changes two things at once (Figure~\ref{fig:ablation-k-sweep}). Larger groups are more likely to contain at least one successful completion, and with binary rewards the same within-group $z$-scoring also makes the grouped update much sharper once a success appears.

We therefore interpret this figure as a joint sensitivity sweep over candidate coverage and grouped-signal sharpness. Across 30 seeds, TPO improves from 8.9\% error at $K{=}4$ to 5.2\% at $K{=}8$, 5.1\% at $K{=}16$, 2.6\% at $K{=}32$, and 0.36\% at $K{=}64$. GRPO is weaker and less monotonic: 19.4\%, 19.8\%, 9.2\%, 4.4\%, and 5.6\% across the same sweep. Under this combined change, TPO behaves more smoothly across $K$.

\begin{figure}[t]
\centering
\includegraphics[width=\textwidth]{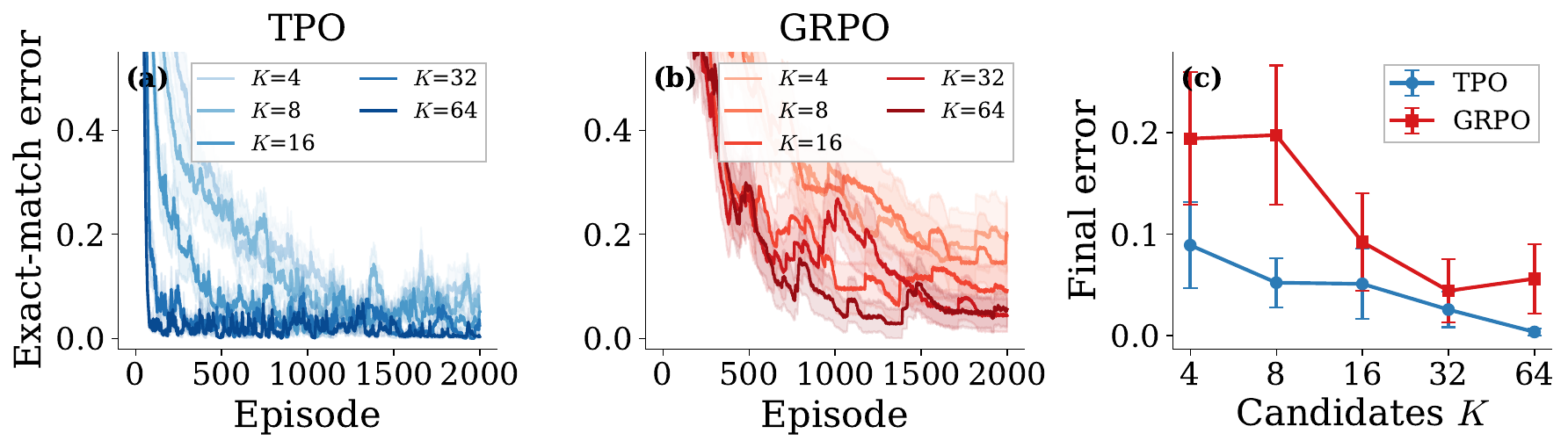}
\caption{\textbf{Group-size sensitivity sweep} ($H{=}8$, $V{=}2$, epochs${=}4$). (a)~TPO learning curves: steady improvement as $K$ grows, with the strongest performance at $K{=}64$. (b)~GRPO learning curves: larger groups help, but performance remains less stable and less monotonic. (c)~Final error vs.\ $K$: TPO improves from 8.9\% at $K{=}4$ to 0.36\% at $K{=}64$; GRPO improves from 19.4\% at $K{=}4$ to 4.4\% at $K{=}32$ and then worsens slightly at $K{=}64$ (5.6\%). 30 seeds, shading/bars $\pm$1 s.e.}
\label{fig:ablation-k-sweep}
\end{figure}

\paragraph{Zero-variance masking.}
If zero-variance groups were simply dead weight, an obvious intervention would be to mask them explicitly. We test ``GRPO (masked),'' which zeros the loss for any group where all $K$ candidates receive the same reward (Figure~\ref{fig:ablation-grpo-fix}).

In the 30-seed aggregate, masking is strongly harmful: final error rises from 6.3\% to 29.7\%. This suggests that these groups are not just junk to delete. In the multi-epoch setting, once informative groups have moved the shared policy, revisiting zero-variance groups can provide a useful anchor back toward the rollout snapshot. Removing them therefore makes GRPO markedly worse, and still does not approach TPO, which reaches 0.05\% in the same setting.

\begin{figure}[t]
\centering
\includegraphics[width=\textwidth]{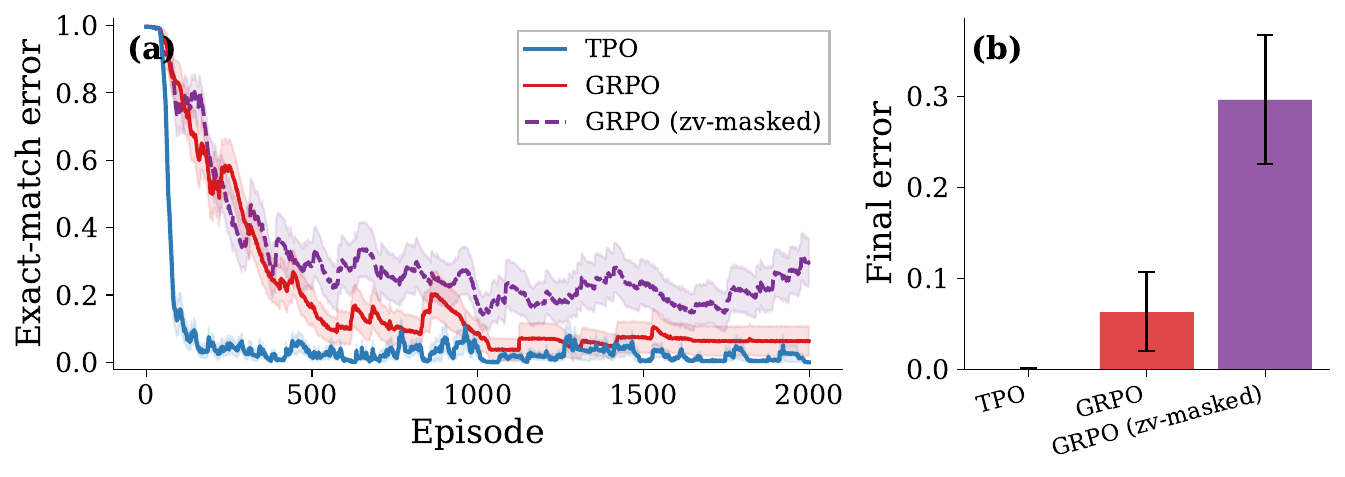}
\caption{\textbf{Zero-variance masking} ($H{=}8$, $V{=}2$, $K{=}32$, epochs${=}4$). (a)~Learning curves: GRPO~(zv-masked) is substantially worse than both GRPO and TPO. (b)~Final error: masking increases GRPO from 6.3\% to 29.7\%, while TPO reaches 0.05\% without any masking. 30 seeds, shading/bars $\pm$1 s.e.}
\label{fig:ablation-grpo-fix}
\end{figure}

\subsection{Does TPO extract more from rare informative batches across epochs?}
\label{sec:analysis-multiepoch}

TPO's fixed target $q$ provides a stable attractor across gradient epochs: the same batch can be reused safely without the trust-region issues that cause DG to diverge with multiple epochs (Appendix~\ref{app:dg-multiepoch}). Under terminal reward, where informative batches are rare, extracting maximum learning from each one is critical.

\begin{figure}[t]
\centering
\includegraphics[width=\textwidth]{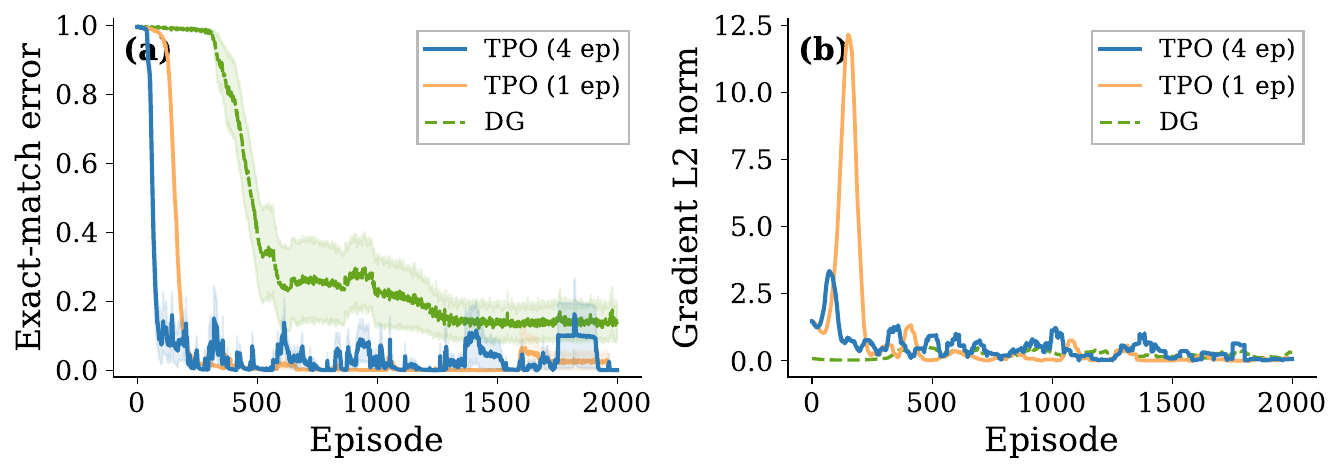}
\caption{\textbf{Multi-epoch extraction} ($H{=}8$, $V{=}2$, $K{=}32$). (a)~Error curves: TPO with 4 gradient epochs reaches 0.2\% error at episode~400 while TPO with 1~epoch is at 1.1\%, roughly $5\times$ faster. Both eventually converge to $<$0.1\%. DG, limited to a single epoch, plateaus at 14\%. (b)~Gradient norms: TPO (4~ep) gradient decays fastest; TPO (1~ep) shows a delayed spike and slower decay; DG's gradient stays low but persistent.}
\label{fig:ablation-multiepoch}
\end{figure}

Figure~\ref{fig:ablation-multiepoch} compares TPO with 4 gradient epochs versus 1. At episode~400, TPO (4~epochs) has reached 0.2\% error while TPO (1~epoch) is at 1.1\%, roughly $5\times$ faster early convergence. Both eventually reach $<$0.1\%, confirming that multi-epoch extraction primarily accelerates learning rather than enabling it. DG, limited to a single epoch, plateaus at 14\%.

\paragraph{Epoch-count ablation.}
We sweep $\{1,2,4,8,16\}$ gradient epochs for both TPO and GRPO (Figure~\ref{fig:ablation-epoch-sweep}). TPO remains stable across the entire range: final error stays below 2.3\% everywhere and is already near zero at 1 and 4~epochs (0.02\% and 0.05\%). GRPO remains strongly non-monotonic: 1~epoch reaches 4.3\%, 2~epochs degrades to 37.6\%, 4~epochs improves to 6.3\%, 8~epochs reaches 3.3\%, and 16~epochs reaches 1.1\%. GRPO can reach low error at the right epoch count, but is much more sensitive to this choice.

\begin{figure}[t]
\centering
\includegraphics[width=\textwidth]{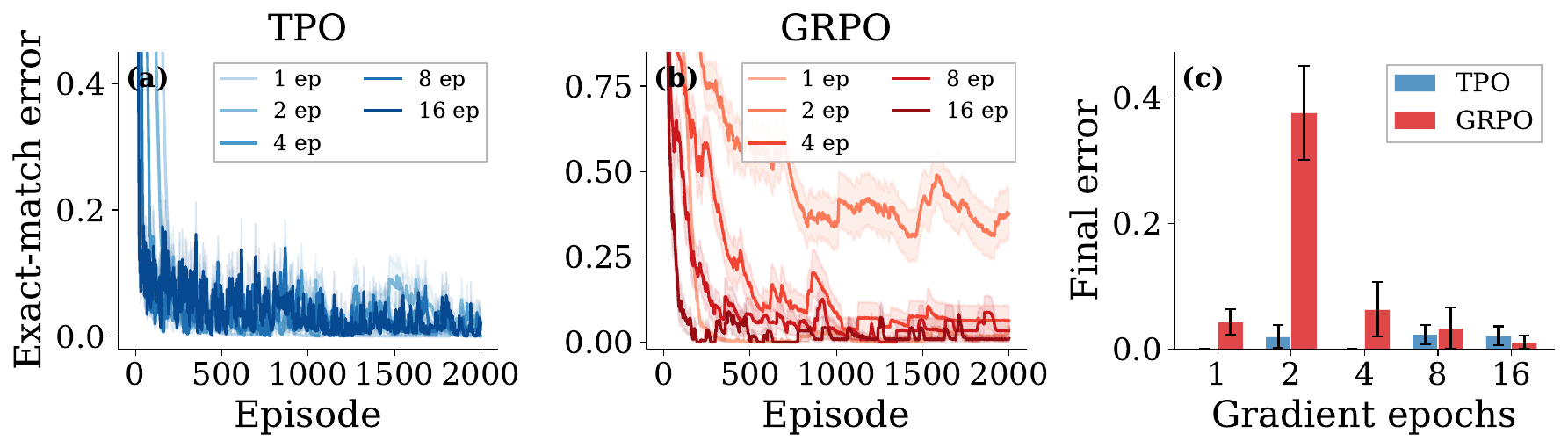}
\caption{\textbf{Epoch-count ablation} ($H{=}8$, $V{=}2$, $K{=}32$). (a)~TPO learning curves across epoch counts: all converge smoothly and remain low-error throughout. (b)~GRPO learning curves: 2~epochs is the worst setting, while 8 and 16~epochs recover strongly. (c)~Final error comparison: TPO stays below 2.3\% everywhere; GRPO is strongly non-monotonic (37.6\% at 2~epochs, 1.1\% at 16~epochs). 30 seeds, shading $\pm$1 s.e.}
\label{fig:ablation-epoch-sweep}
\end{figure}

\medskip\noindent
No single property explains TPO's sparse-reward advantage. The gradient norm collapses as the policy approaches its target, performance degrades smoothly rather than abruptly when $K$ or epoch count varies, and multi-epoch reuse works without careful tuning. These properties reinforce each other and are absent from the baselines.

\section{Related work}
\label{sec:related}

\textbf{Target-matching and mirror-descent methods.}
TPO's target (Eq.~\ref{eq:target}) is the closed-form solution to
$\arg\min_q \mathrm{KL}(q \,\|\, p^{\text{old}}) - \tfrac{1}{\eta}\,\mathbb{E}_q[u]$
restricted to $K$ candidates.
The closest relatives are REPS~\citep{peters2010reps}, MPO~\citep{abdolmaleki2018mpo}, and V-MPO~\citep{song2020vmpo}, which use the same exponential-tilting step but require a critic or value estimate to supply the improvement signal.
AWR~\citep{peng2019awr} also uses KL-regularized improvement weights but treats each sample's $\exp(A/\beta)$ as a fixed scalar on its log-likelihood, so its gradient does not self-extinguish at the target.
TPO's distinguishing property is that the finite scored candidate set provides the target in closed form without a critic or inner optimization loop, and its gradient $p^\theta{-}q$ vanishes once the target is matched.
MDPO~\citep{tomar2020mdpo} gives a mirror-descent perspective on the same family.
More generally, TPO can be read as a KL-regularized policy-improvement operator on the sampled candidate simplex rather than the full action space~\citep{kakade2001natural,geist2019regularized}.

\textbf{Group-based policy-gradient methods.}
RLOO~\citep{ahmadian2024rloo} and GRPO~\citep{shao2024deepseekmath} also score multiple candidates per context but convert them into per-sample scalar weights inside a policy-gradient objective.
TPO instead builds a target distribution on the candidate simplex and fits the policy to it.
Recent GRPO variants address specific failure modes while remaining scalar-weighted PG methods: Dr.~GRPO~\citep{liu2025drgrpo} removes a difficulty bias from within-group $\sigma$-normalization~\citep{murphy2025reinforcementlearningoverview}; DAPO~\citep{yu2025dapo} uses asymmetric clipping to prevent entropy collapse; GSPO~\citep{zhang2025gspo} fixes a per-token importance-ratio mismatch when rewards are trajectory-level.
TPO replaces the scalar-weighted update with a single target distribution over the group, avoiding importance ratios and clipping entirely. Because it still standardizes within-group scores, however, low-variance difficulty-bias effects can remain in principle, as discussed in Section~\ref{sec:discussion}.

\textbf{Single-sample policy-gradient methods.}
REINFORCE~\citep{williams1992simple}, TRPO~\citep{schulman2015trpo}, PPO~\citep{schulman2017ppo},
REINFORCE++~\citep{hu2025reinforceplusplus}, and ReMax~\citep{li2024remax} all assign scalar advantage weights to sampled actions.
ReMax removes the value model and uses a greedy-decode baseline for variance reduction, yielding large memory and speed gains over PPO, but the gradient remains the standard score-function estimator.
DG~\citep{osband2026dg} corrects gradient misallocation \emph{across contexts} via sigmoid gating;
TPO addresses misallocation \emph{within} a context's candidate set.
The two are complementary and can be composed.

\textbf{Regression- and preference-based methods.}
REBEL~\citep{gao2024rebel} reduces RL to iterative least-squares regression on reward \emph{differences} between paired completions, generalizing NPG with strong agnostic regret bounds.
Both REBEL and TPO construct a target from rewards and the behavior policy, but differ in loss and structure: REBEL uses squared loss on log-probability ratios over pairs; TPO uses cross-entropy on a distribution over a candidate group.
PMPO~\citep{abdolmaleki2025pmpo} is the closest target-matching method: it partitions candidates into accepted/rejected sets and regularizes toward a frozen $\pi_{\text{ref}}$, whereas TPO keeps a single soft target over the full group and anchors only to $\pi_{\text{old}}$.
Offline pairwise methods (DPO~\citep{rafailov2023dpo}, KTO~\citep{ethayarajh2024kto}, IPO~\citep{azar2024ipo}) are more distant, as TPO is online, setwise, and scorer-agnostic.

\textbf{Objective-level corrections.}
MaxRL~\citep{tajwar2026maxrl} changes \emph{which} objective is optimized (higher-order corrections under binary rewards).
GDPO~\citep{liu2026gdpogrouprewarddecouplednormalization} and MT-GRPO~\citep{ramesh2026multitaskgrporeliablellm} correct GRPO's objective for multi-reward and multi-task settings, respectively: GDPO decouples per-reward normalization to prevent advantage collapse, while MT-GRPO introduces robustness-aware task reweighting.
TPO is orthogonal, changing \emph{how} within-context signals become updates; all four corrections are complementary (see Section~\ref{sec:discussion}).

\textbf{Off-policy and asynchronous training.}
Large-scale RL pipelines decouple rollout generation from parameter updates, introducing staleness and engine mismatch.
ScaleRL~\citep{khatri2025artscalingreinforcementlearning} systematically studies this regime, showing that the degree of off-policy-ness modulates compute efficiency without shifting the asymptotic performance ceiling, and proposes truncated importance sampling to stabilize training.
IcePop~\citep{lingteam2025stepevolvesscalingreinforcement} addresses a distinct source of off-policy-ness: probability discrepancies between inference and training engines (especially in MoE models), which compound across iterations; it masks token-level gradients whose engine-ratio falls outside a calibrated window.
TPO's within-context correction is orthogonal to both and can be composed with either off-policy strategy.

\begin{table*}[t]
\caption{\textbf{Comparison of policy optimization methods.}
Group: update structurally compares candidates within a context. Fixed ref.: requires a frozen reference model beyond $\pi_{\text{old}}$.}
\label{tab:method-comparison}
\centering
\small
\setlength{\tabcolsep}{10pt}
\begin{tabular}{@{}llccc@{}}
\toprule
\textbf{Method} & \textbf{Update rule} & \textbf{Group} & \textbf{Critic} & \textbf{Fixed ref.} \\
\midrule
REINFORCE & PG + baseline & \ding{55} & \ding{55} & \ding{55} \\
REINFORCE++ & PG + per-token KL baseline & \ding{55} & \ding{55} & \ding{51} \\
ReMax & PG + greedy baseline & \ding{55} & \ding{55} & \ding{51} \\
TRPO & PG + KL trust region & \ding{55} & \ding{55} & \ding{55} \\
PPO & Clipped PG surrogate & \ding{55} & \ding{55} & \ding{55} \\
DG & Sigmoid-gated PG & \ding{55} & \ding{55} & \ding{55} \\
MDPO & PG + mirror-descent KL & \ding{55} & \ding{55} & \ding{55} \\
\midrule
RLOO & PG + leave-one-out baseline & \ding{51} & \ding{55} & \ding{55} \\
GRPO & Clipped PG + group adv. & \ding{51} & \ding{55} & \ding{55} \\
REBEL & Sq.\ loss on reward diffs & \ding{51} & \ding{55} & \ding{55} \\
PMPO & Weighted lik.\ + KL to $\pi_{\text{ref}}$ & \ding{51} & \ding{55} & \ding{51} \\
AWR & Regress to $\exp(A/\beta)$ weights & \ding{55} & \ding{51} & \ding{55} \\
MPO / V-MPO & $q{\propto}\pi_{\text{old}}\exp(\text{signal}/\eta)$; fit $\pi{\to}q$ & \ding{51} & \ding{51} & \ding{55} \\
MaxRL & Higher-order obj.\ correction & \ding{51} & \ding{55} & \ding{55} \\
\midrule
\textbf{TPO} & $q{\propto}p_{\text{old}}\exp(u)$; CE to $q$ & \ding{51} & \ding{55} & \ding{55} \\
\bottomrule
\end{tabular}
\end{table*}

\section{Limitations}
\label{sec:discussion}

\textbf{Candidate quality and group-based costs.}
TPO can only redistribute probability over the candidates it is given. If the sampled set is low-diversity or uniformly poor, the target is correspondingly uninformative. In discrete-action settings where all actions can be scored in a single forward pass, the $K$-candidate group adds no extra environment interactions; in sequence settings without a critic, TPO requires $K$ rollouts per context just as GRPO does. It may use those rollouts better, but does not remove the cost. More aggressive rollout reuse would move TPO into a genuinely off-policy regime, where Retrace- or V-trace-style corrections may become necessary~\citep{munos2016safe,espeholt2018impala}.

\textbf{Score standardization is helpful but not free.}
Standardizing scores gives TPO a stable scale across tasks and largely removes the need to tune temperature, but it can also amplify small numerical differences when the within-group variance is tiny. For instance, a group where one candidate scores $0.001$ and the rest score $0$ produces a very sharp target after $z$-scoring. This is the same difficulty-bias mechanism identified for GRPO~\citep{liu2025drgrpo,murphy2025reinforcementlearningoverview}. A more robust treatment of low-variance groups would help in practice.

\textbf{Scale of evaluation.}
Our LLM-scale experiments (Section~\ref{sec:exp-llm-rlvr}) use 1.5--1.7B parameter models on three tasks. Testing on larger models (7B+) and harder benchmarks (MATH, AIME) remains future work; the main open question is whether TPO's relative gains persist at larger scale.

\section{Conclusion}
\label{sec:conclusion}

TPO replaces scalar-weighted policy gradients with a single design choice: build a target distribution on the scored candidate set and fit the policy to it by cross-entropy. Across every setting we tested (tabular bandits, neural bandits, dense- and sparse-reward transformers, and billion-parameter LLM RLVR), TPO matches PG, PPO, DG, and GRPO on dense-reward tasks and substantially outperforms them under sparse reward. Separating \emph{what} redistribution is desired from \emph{how} the optimizer realizes can make the update more robust. We plan to test TPO on larger models.

\begin{ack}
We thank Srijan Patel, Zhengyao Jiang, and Ian Osband for discussions and feedback.

\end{ack}

\bibliographystyle{plainnat}
\bibliography{references}

\appendix

\section{Score standardization}
\label{app:whitening}

Given raw scores $s_1, \dots, s_K$, the standardized scores used throughout the paper are
\begin{equation}
\label{eq:skill-white}
u_i =
\begin{cases}
\dfrac{s_i - \bar{s}}{\sigma(s)} & \text{if } \sigma(s) > 0, \\
0 & \text{if } \sigma(s) = 0,
\end{cases}
\end{equation}
where $\bar{s} = \frac{1}{K}\sum_j s_j$ and
\[
\sigma(s)=\sqrt{\frac{1}{K}\sum_{j=1}^K (s_j-\bar{s})^2}
\]
is the within-group \emph{population} standard deviation. Equivalently, $u_i$ is the within-group z-score of $s_i$, with the convention that every coordinate is set to zero when $\sigma(s)=0$.

\section{Multi-context tabular weighting derivation}
\label{app:tabular-multi-weight}

This appendix derives the effective per-context coefficients for the multi-context tabular bandit in Section~\ref{sec:exp-tabular-multi}. The derivation is in \emph{logit space}, matching the experiment and \citet{osband2026dg}: the policy in each context is a softmax over explicit logits. To avoid overloading $K$ from the main text, let $A$ denote the number of actions in this bandit. In context $n$, let $y_n$ be the correct action, $\pi_n$ the current policy over $A$ actions, $p_n=\pi_n(y_n)$, and $e_{y_n}$ the one-hot vector for the correct action. Define
\[
v_n = e_{y_n} - \pi_n.
\]
Because the reward is one-hot, all exact updates considered in the multi-context experiment lie along this same direction; the methods differ only in the scalar coefficient multiplying $v_n$.
For TPO, this scalar form appears only after first constructing the target $q_n$ and then simplifying the target-matching update $q_n-\pi_n$ in this special one-hot setting.

\paragraph{CE.}
The cross-entropy oracle uses
\[
g_n^{\mathrm{CE}} = e_{y_n} - \pi_n = v_n,
\]
so
\[
\beta_{\mathrm{CE}}(p_n)=1.
\]

\paragraph{DG.}
In the exact population limit, with baseline $b{=}0$, DG contributes
\[
g_n^{\mathrm{DG}}
= \sum_a \pi_n(a)\, r_n(a)\, \sigma\!\left(\frac{r_n(a)(-\log \pi_n(a))}{\eta}\right)\, (e_a-\pi_n),
\]
where $r_n(a)=\mathbf{1}\{a=y_n\}$. Since only the correct action has nonzero reward,
\[
g_n^{\mathrm{DG}}
= p_n \, \sigma\!\left(\frac{-\log p_n}{\eta}\right)\, (e_{y_n}-\pi_n).
\]
Therefore
\[
\beta_{\mathrm{DG}}(p_n)=p_n \sigma\!\left(\frac{-\log p_n}{\eta}\right).
\]
With the default $\eta=1$ used in the experiment,
\[
\beta_{\mathrm{DG}}(p_n)=\frac{p_n}{1+p_n}.
\]

\paragraph{GRPO.}
Within context $n$, rewards are Bernoulli with mean $p_n$ and standard deviation
\[
\sigma_n=\sqrt{p_n(1-p_n)}.
\]
Standardizing rewards gives advantage
\[
A_n(y_n)=\frac{1-p_n}{\sigma_n},
\qquad
A_n(a\neq y_n)=\frac{-p_n}{\sigma_n}.
\]
The exact population GRPO update is
\[
g_n^{\mathrm{GRPO}}
= \sum_a \pi_n(a)\, A_n(a)\, (e_a-\pi_n).
\]
Substituting the two advantage values and simplifying yields
\[
g_n^{\mathrm{GRPO}}
= \frac{p_n}{\sigma_n}\, (e_{y_n}-\pi_n)
= \sqrt{\frac{p_n}{1-p_n}}\, (e_{y_n}-\pi_n),
\]
so
\[
\beta_{\mathrm{GRPO}}(p_n)=\sqrt{\frac{p_n}{1-p_n}}.
\]

\paragraph{TPO.}
For the one-hot score vector $s=e_{y_n}$, the mean is $\bar{s}=1/A$ and the population standard deviation is $\sigma(s)=\sqrt{A-1}/A$. Using Eq.~\ref{eq:skill-white},
\[
u_{y_n}
= \frac{1-1/A}{\sqrt{A-1}/A}
= \sqrt{A-1},
\qquad
u_{a\neq y_n}
= \frac{-1/A}{\sqrt{A-1}/A}
= -\frac{1}{\sqrt{A-1}}.
\]
For the $A{=}10$ experiment, this is the z-score of $(1,0,\dots,0)$: $\bar{s}=1/10$ and
\[
\sigma(s)^2
= \frac{1}{10}\left(1-\frac{1}{10}\right)^2
+ \frac{9}{10}\left(0-\frac{1}{10}\right)^2
= \frac{9}{100},
\qquad
\sigma(s)=\frac{3}{10}.
\]
Hence
\[
u_{y_n}=\frac{1-1/10}{3/10}=3,
\qquad
u_{a\neq y_n}=\frac{0-1/10}{3/10}=-\frac{1}{3}.
\]
TPO still starts by forming the target
\[
q_n(a) \propto \pi_n(a)\exp(u_a),
\]
which multiplies the correct-vs.-incorrect odds by a fixed factor
\[
\lambda
= \exp\!\left(u_{y_n}-u_{a\neq y_n}\right)
= \exp\!\left(\frac{A}{\sqrt{A-1}}\right).
\]
For $A=10$, therefore, $\lambda=\exp(10/3)\approx 28$. The TPO target therefore satisfies
\[
q_n(y_n)=\frac{\lambda p_n}{1-p_n+\lambda p_n},
\qquad
q_n(a\neq y_n)=\frac{\pi_n(a)}{1-p_n+\lambda p_n}.
\]
The TPO loss gradient is $\pi_n-q_n$, so the corresponding gradient-descent update direction is $g_n^{\mathrm{TPO}}=q_n-\pi_n$. This simplifies to
\[
g_n^{\mathrm{TPO}}
= \frac{p_n(\lambda-1)}{1-p_n+\lambda p_n}\, (e_{y_n}-\pi_n).
\]
Thus
\[
\beta_{\mathrm{TPO}}(p_n)=\frac{p_n(\lambda-1)}{1-p_n+\lambda p_n}.
\]
In other words, $\beta_{\mathrm{TPO}}$ is not a different definition of TPO; it is the closed-form coefficient obtained after eliminating $q_n$ from the update $q_n-\pi_n$ in this tabular one-hot case.

\paragraph{Interpretation.}
All four updates share the same within-context direction $e_{y_n}-\pi_n$ and differ only in their cross-context weight $\beta(p_n)$. CE weights every context equally. DG and GRPO place relatively more weight on contexts with larger $p_n$, so under a normalized step they spend more update budget on already-easy contexts. TPO's coefficient is much flatter in $p_n$ and therefore closer to CE's equal-weight allocation. For example, at $p_n=0.1$ and $A=10$, $\beta_{\mathrm{TPO}}=0.73$, versus $0.09$ for DG and $0.33$ for GRPO.

\section{MNIST single-example logit updates}
\label{app:mnist-logit}

This appendix derives the expected logit-space updates for the MNIST contextual bandit in Section~\ref{sec:exp-discrete}, showing what information each loss preserves from a single bandit sample. Consider one labeled example $(x,y)$ with logits $z$, policy $\pi=\pi(\cdot\mid x)$ over the 10 classes, correct-class probability $p=\pi_y$, and one-hot basis vectors $e_i$. The supervised cross-entropy direction on this example is
\[
v = e_y - \pi.
\]
All expectations below are over the sampled action $a \sim \pi(\cdot\mid x)$. Throughout this appendix, $g$ denotes the \emph{gradient-descent update direction} in logit space, i.e.\ the negative of the loss gradient. These are the directions induced by the implemented surrogate losses, with scalar coefficients such as baselines, standardized rewards, gates, and target distributions treated as stop-gradient constants exactly as in the code.

\paragraph{PG.}
The MNIST baseline is
\[
b = \sum_{i=1}^{10} \pi_i^2,
\]
so the per-sample advantage is $A(a)=\mathbf{1}\{a=y\}-b$. The expected REINFORCE update is
\[
g^{\mathrm{PG}}
= \mathbb{E}\!\left[A(a)\,(e_a-\pi)\right]
= p(e_y-\pi)
= p\,v.
\]
The baseline term disappears because $\mathbb{E}[e_a-\pi]=0$.

\paragraph{Single-sample GRPO.}
In the implemented MNIST variant, rewards are standardized across the minibatch:
\[
A_B(a)=\frac{\mathbf{1}\{a=y\}-\mu_B}{\sigma_B},
\]
where $\mu_B$ and $\sigma_B$ are the minibatch reward mean and standard deviation. Conditioning on the realized minibatch statistics $(\mu_B,\sigma_B)$ for one example, the expected update is
\[
g^{\mathrm{GRPO}\mid \mu_B,\sigma_B}
= \mathbb{E}\!\left[A_B(a)\,(e_a-\pi)\right]
= \frac{p}{\sigma_B}(e_y-\pi)
= \frac{p}{\sigma_B}\,v.
\]
Thus this single-sample MNIST variant is REINFORCE with batch-standardized rewards: the exact minibatch update couples examples through $\mu_B$ and $\sigma_B$, but it introduces no new within-example geometry.

\paragraph{DG.}
DG uses the same advantage $A(a)=\mathbf{1}\{a=y\}-b$ but gates it by surprisal. Since
\[
A(y)=1-b,
\qquad
A(j\neq y)=-b,
\]
the exact expected logit update is
\[
g^{\mathrm{DG}}
= p(1-b)\,\sigma\!\left(\frac{(1-b)\log(1/p)}{\eta}\right)(e_y-\pi)
{}- b\sum_{j\neq y}\pi_j\,
\sigma\!\left(-\frac{b\log(1/\pi_j)}{\eta}\right)(e_j-\pi).
\]
In general this need not be collinear with $v$: the update depends on how probability mass is distributed across the wrong classes. Under the symmetric one-vs-rest approximation $\pi_j=q=(1-p)/9$ for all $j\neq y$, it collapses to
\[
g^{\mathrm{DG}} = \beta_{\mathrm{DG}}^{\mathrm{sym}}(p)\,v,
\]
with
\[
\beta_{\mathrm{DG}}^{\mathrm{sym}}(p)
= p(1-b)\,\sigma\!\left(\frac{(1-b)\log(1/p)}{\eta}\right)
 + p b\,\sigma\!\left(-\frac{b\log(1/q)}{\eta}\right),
\]
where $b = p^2 + 9q^2$.

\paragraph{TPO.}
TPO builds a target from the sampled action. The sampled score vector is
\[
s = A(a)\,e_a.
\]
Because $s$ has exactly one nonzero coordinate, z-scoring over $K=10$ classes maps a positive sample to $u_a=3$ and $u_{i\neq a}=-1/3$, and a negative sample to the sign-flipped pattern. After standardization, only the sign of $A(a)$ matters. Define
\[
\lambda = \exp\!\left(\frac{10}{3}\right) \approx 28,
\]
the corresponding correct-vs-incorrect reweighting factor for $K=10$ classes, since $\lambda=\exp(3-(-1/3))$.

If the sampled action is correct ($a=y$), the target is
\[
q_i^{+} \propto \pi_i \exp(u_i^{+}),
\]
with $u_y^{+}=3$ and $u_{j\neq y}^{+}=-1/3$. This gives
\[
q_y^{+}=\frac{\lambda p}{1-p+\lambda p},
\qquad
q_{j\neq y}^{+}=\frac{\pi_j}{1-p+\lambda p},
\]
so the induced logit update is
\[
g^{+}=q^{+}-\pi
= \beta_{+}(p)\,(e_y-\pi),
\qquad
\beta_{+}(p)=\frac{p(\lambda-1)}{1-p+\lambda p}.
\]

If the sampled action is an incorrect class $j\neq y$, standardization flips sign: the sampled wrong class receives $u_j^-=-3$ and every other class receives $u_i^-=1/3$. The target is then
\[
q_j^{-}=\frac{\pi_j}{\lambda(1-\pi_j)+\pi_j},
\qquad
q_{i\neq j}^{-}=\frac{\lambda\pi_i}{\lambda(1-\pi_j)+\pi_j},
\]
which yields
\[
g^{-(j)}=q^{-}-\pi
= \gamma(\pi_j)\,(\pi-e_j),
\qquad
\gamma(r)=\frac{r(\lambda-1)}{\lambda(1-r)+r}.
\]
Taking expectation over the sampled action gives
\[
g^{\mathrm{TPO}}
= p\,g^{+} + \sum_{j\neq y}\pi_j\,g^{-(j)}.
\]
Unlike PG, GRPO, and DG, a success pulls directly toward the label while a failure directly suppresses the sampled wrong class, redistributing that mass across the remaining logits.

Under the symmetric one-vs-rest approximation $\pi_j=q$ for all $j\neq y$, TPO also collapses to a scalar multiple of $v$:
\[
g^{\mathrm{TPO}} = \beta_{\mathrm{TPO}}^{\mathrm{sym}}(p)\,v,
\qquad
\beta_{\mathrm{TPO}}^{\mathrm{sym}}(p)
= p\,\beta_{+}(p) + p\,\gamma(q).
\]

\paragraph{Group PG.}
Our same-signal scalar ablation keeps the same sampled score vector $s=A(a)e_a$ as TPO, but replaces target matching with scalar-weighted REINFORCE using the sampled standardized score $u_a$. For $K=10$, the sampled coordinate has standardized value $u_a=3$ when $a=y$ and $u_a=-3$ when $a\neq y$; the unsampled coordinates do not enter the scalar-weighted loss. Therefore
\[
g^{\mathrm{GroupPG}}
= 3p(e_y-\pi) - 3\sum_{j\neq y}\pi_j(e_j-\pi)
= 6p(e_y-\pi)
= 6g^{\mathrm{PG}}.
\]
Thus Group PG holds the sampled signal fixed but discards TPO's target structure; in expectation it collapses back to a rescaled one-vs-rest PG update.

\paragraph{Interpretation.}
The derivation isolates what information survives from a single bandit sample. PG, conditional single-sample GRPO, and the same-signal scalar ablation Group PG all reduce to one-vs-rest directions in expectation, so they only preserve a scalar ``correct versus incorrect'' signal. DG and TPO condition on the sampled action, so in general they depend on the detailed distribution of wrong-class mass. When the wrong classes are nearly symmetric, both reduce to scalar multiples of $e_y-\pi$. Away from that limit, TPO retains a particularly useful failure update: it explicitly suppresses the sampled wrong class and redistributes that mass elsewhere. Therefore TPO should help most when the model's mistakes are concentrated on one or a few confusing alternatives, and least when the wrong-class mass is diffuse. Section~\ref{sec:exp-discrete} tests exactly this prediction.

\section{Temperature robustness}
\label{app:temperature}

Score standardization sets an effective temperature of $\eta = 1$: the target distribution becomes $q_i \propto p_i^{\text{old}} \cdot \exp(u_i / \eta)$ with $\eta = 1$. To test sensitivity, we sweep $\eta \in \{0.25, 0.5, 1, 2, 4\}$ on the token reversal task (reverse copy, $V{=}2$, $H{=}10$, $B{=}100$, $K{=}8$, 10 seeds).

\begin{figure}[h]
\centering
\includegraphics[width=0.55\textwidth]{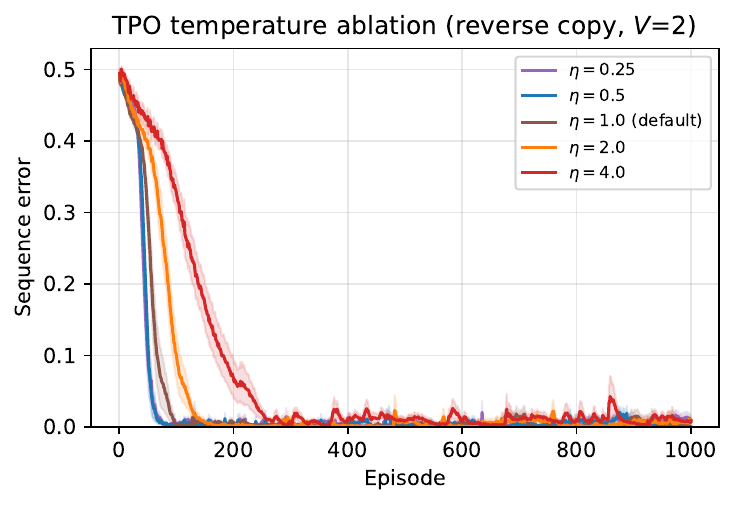}
\caption{\textbf{TPO temperature ablation.} All values in $[0.25, 2]$ converge within 141 episodes; only $\eta{=}4$ is meaningfully slower. Performance is robust across a 16$\times$ range.}
\label{fig:eta_sweep}
\end{figure}

Table~\ref{tab:eta_sweep} reports steps to 1\% error. All values from 0.25 to 2.0 reach 1\% within 141 episodes; only $\eta{=}4$ degrades substantially. The default $\eta{=}1$ sits in the middle of a wide basin of good performance, consistent with the finding of \citet{osband2026dg}, who independently report that $\eta{=}1$ is robust for both DG and MPO across MNIST and DM Control.

\begin{table}[h]
  \caption{\textbf{TPO temperature ablation} on token reversal (reverse copy, $V{=}2$). Performance is robust across a wide range; only $\eta{=}4$ degrades substantially.}
  \label{tab:eta_sweep}
  \centering
  \small
  \begin{tabular}{lccc}
    \toprule
    $\eta$ & Final error (\%) & Steps to 1\% \\
    \midrule
    0.25 & 1.0 & 72 \\
    0.50 & 0.0 & 67 \\
    1.00 (default) & 0.7 & 96 \\
    2.00 & 1.0 & 141 \\
    4.00 & 0.8 & 260 \\
    \bottomrule
  \end{tabular}
\end{table}

\section{Multi-epoch DG instability}
\label{app:dg-multiepoch}

PPO, GRPO, and TPO all include mechanisms that limit or anchor the policy relative to the rollout-time or reference policy: PPO clips the importance-weight ratio, GRPO adds a KL penalty, and TPO fits an explicit target distribution whose construction is KL-anchored to $p^{\text{old}}$ (cf.\ the EM control M-step in MPO~\citep{abdolmaleki2018mpo}). In our experiments, these stabilizing mechanisms made multi-epoch reuse substantially more stable than DG and improved data extraction from each rollout batch.

DG lacks such a constraint because it is explicitly designed as a ``drop-in replacement for standard policy gradients that requires no importance ratios''~\citep{osband2026dg}, modulating gradient magnitude via sigmoid gating but not bounding the per-step policy shift. When we rerun DG with the same 4 gradient epochs used by PPO, GRPO, and TPO, the behavior becomes highly sensitive to epoch count. On a reverse-copy transformer RLVR benchmark with terminal reward, 4-epoch DG finishes at 48.3\% error versus 2.0\% for the standard 1-epoch update (Figure~\ref{fig:dg-multiepoch}(a)). Across the eight prompt-matched token-reversal variants from Section~\ref{sec:exp-variations}, 4-epoch DG is worse in 7 of 8 settings (Figure~\ref{fig:dg-multiepoch}(b,c)), with the largest regressions on the sequential tasks: flip rises from 0.07\% to 4.56\% and reverse flip from 0.00\% to 0.82\%. The only exception is reverse copy with sequential reward, where 4 epochs improves slightly (0.35\% to 0.05\%).

\begin{figure*}[t]
\centering
\includegraphics[width=\textwidth]{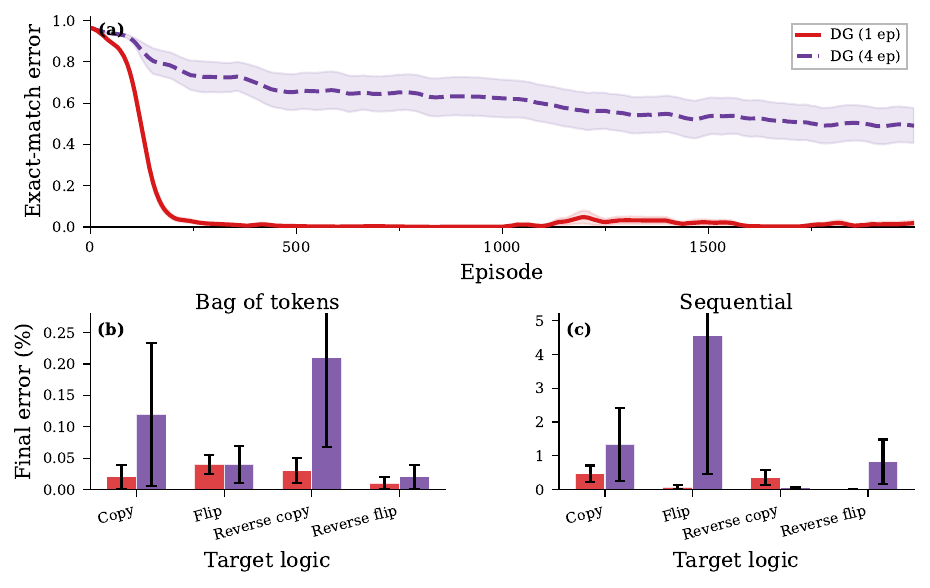}
\caption{\textbf{DG epoch sensitivity across sparse- and dense-reward transformer tasks.} (a) Reverse-copy transformer RLVR with terminal reward, 20 seeds: reusing each rollout batch for 4 DG gradient epochs keeps the error high (48.3\% final) while the standard 1-epoch DG update reaches 2.0\%. (b,c) Final error on the eight prompt-matched token-reversal variants from Section~\ref{sec:exp-variations} ($H{=}10$, $V{=}2$, $K{=}8$ token candidates, 10 seeds), split by reward type. DG with 4 epochs is worse in 7 of 8 settings, with the largest regressions on the sequential tasks. Shading and error bars show $\pm 1$ s.e.}
\label{fig:dg-multiepoch}
\end{figure*}

We therefore run DG with a single gradient epoch per rollout batch throughout all experiments. This is the most favorable setting for DG and is consistent with \citet{osband2026dg}, who use DG as a single-step on-policy update throughout their experiments.

\section{GRPO baseline configuration}
\label{app:grpo-baseline}

Our GRPO baseline uses the standard PPO-style clipped surrogate with group-relative ($z$-scored) advantages~\citep{shao2024deepseekmath}, augmented with a reverse-KL penalty ($\beta{=}0.04$) to the rollout policy. In the original DeepSeekMath setup this KL is taken to a reference policy (e.g.\ the SFT checkpoint), while iterative GRPO variants can also use the current policy as the reference; in our controlled experiments, which train from scratch with no separate reference model, we therefore penalize divergence from the rollout snapshot.

This is a deliberate strengthening of the baseline: removing the KL term ($\beta{=}0$) causes GRPO to collapse under sparse terminal reward, with error increasing over training rather than decreasing (Section~\ref{sec:exp-rlvr}, Table~\ref{tab:rlvr}). The KL penalty stabilizes multi-epoch reuse by preventing the policy from drifting too far from the data that generated the advantages, a role that TPO's cross-entropy-to-target objective fulfills structurally without requiring an explicit penalty.

\section{LLM RLVR implementation details}
\label{app:llm-rlvr-details}

All LLM RLVR experiments use the \texttt{verl} stack~\citep{sheng2024verl} with AdamW at learning rate $10^{-5}$, batch size~16, and $4\times$A100-80GB GPUs. GSM8K uses exact-match rewards; graph coloring uses quasi-binary native task scores; Knights \& Knaves uses partial-credit scores. For GSM8K we add LoRA (rank~32) and a KL penalty ($\lambda_{\text{KL}} = 10^{-3}$) to both TPO and GRPO. The paired runs are otherwise identical, differing only in the policy loss: TPO uses Eq.~\ref{eq:loss}; GRPO uses the clipped surrogate with $z$-scored advantages.

\end{document}